\documentclass{article} 
\usepackage{iclr2027_conference,times}

\usepackage{hyperref}
\usepackage{url}
\usepackage{algorithm}
\usepackage{algpseudocode}
\usepackage{amsmath}
\usepackage{amssymb}
\usepackage{mathtools}
\usepackage{amsthm}
\usepackage{booktabs}
\usepackage{graphicx}
\graphicspath{{../}}
\usepackage{subfig}
\usepackage{xcolor}
\usepackage{multirow}
\usepackage{bm}
\usepackage{adjustbox}
\usepackage{wrapfig}

\usepackage{amsmath,amsfonts,bm}

\def\eqref#1{equation~\ref{#1}}

\def\1{\bm{1}}

\def\vtheta{{\bm{\theta}}}

\def\vh{{\bm{h}}}

\def\vx{{\bm{x}}}
\def\vy{{\bm{y}}}

\DeclareMathAlphabet{\mathsfit}{\encodingdefault}{\sfdefault}{m}{sl}
\SetMathAlphabet{\mathsfit}{bold}{\encodingdefault}{\sfdefault}{bx}{n}

\theoremstyle{plain}
\newtheorem{theorem}{Theorem}[section]
\newtheorem{proposition}[theorem]{Proposition}
\newtheorem{lemma}[theorem]{Lemma}
\newtheorem{corollary}[theorem]{Corollary}
\theoremstyle{definition}

\newtheorem{assumption}[theorem]{Assumption}
\theoremstyle{remark}
\newtheorem{remark}[theorem]{Remark}

\title{Conformal Prediction-Enabled\\Sequence-level Policy Optimization for\\Reasoning Language Models}

\author{Hui Yang\(^{1}\)\footnotemark[1]\quad
Wan Tian\(^{1}\)\footnotemark[1]\quad
Tao Ren\(^{1}\)\footnotemark[1]\quad
Zhongyi Li\(^{2}\)\quad
Yijie Peng\(^{3}\)\footnotemark[2]\;\;
 \\
\(^{1}\) Peking University\quad
\(^{2}\) Beihang University\quad
\(^{3}\) Nanjing University
}

\iclrfinalcopy 

\begin{document}

\renewcommand{\thefootnote}{\fnsymbol{footnote}} 
\footnotetext[1]{\mbox{\ Equal contribution.}}
\footnotetext[2]{\ Corresponding author: \texttt{pengyijie@nju.edu.cn}.}
\footnotetext[3]{\ Email to: \texttt{yanghui6@stu.pku.edu.cn}.}

\maketitle

\begin{abstract}
Reinforcement learning improves the reasoning capability of large language models through verifiable rewards for mathematical and coding tasks. In token-level methods such as GRPO, clipping acts on individual token ratios while the reward and advantage apply to a complete response. Recent sequence-level methods align these units but still use fixed clipping bands, whose tail probabilities change as the score distribution evolves. We propose \emph{Conformal Prediction-Enabled Sequence-level Policy Optimization} (CSPO), which adaptively calibrates sequence-level clipping bounds using recent training statistics. The resulting band targets a specified marginal tail budget and adapts to policy drift; calibration conditional on response length requires additional assumptions. We bound the probability that a fresh score falls outside the band by the target miscoverage level plus explicit CDF discrepancy and drift terms. Under additional moment conditions, this bound also controls the objective and gradient error relative to the unclipped GSPO power surrogate. Experiments with the Qwen3-1.7B series show improvements over GSPO on difficult mathematical reasoning benchmarks and ALFWorld, with clipping endpoints that adapt during training.
\end{abstract}

\section{Introduction}
Reinforcement learning from feedback has become a key component of large language model post-training. Early reinforcement learning from human feedback (RLHF) pipelines optimize policies against learned reward models derived from human preferences, enabling instruction-following and preference alignment beyond supervised fine-tuning \citep{christiano2017deep,stiennon2020learning,ouyang2022training,bai2022constitutional}. More recently, reinforcement learning with verifiable rewards (RLVR) has emerged as an especially effective paradigm for reasoning tasks, where correctness can often be automatically checked by final-answer matching, symbolic verification, unit tests, or external graders. Such rewards have enabled large-scale reasoning-oriented post-training, as exemplified by recent reasoning models and open RL systems \citep{guo2025deepseekr1,wen2025reinforcement,yu2025dapo}. In this setting, the reward is naturally defined at the \emph{sequence level}: a complete solution is judged by whether the final answer and, in some cases, the reasoning trajectory satisfy the verification criterion.

Most RL post-training methods build on trust-region policy optimization. TRPO constrains the policy update through a KL trust region, while PPO replaces the hard constraint with a clipped surrogate objective that is simpler and more scalable \citep{schulman2015trust,schulman2017proximal}. Clipping limits the incentive to move probability ratios away from one, reducing the risk of overfitting sampled trajectories during repeated optimizer updates. For language models, GRPO simplifies PPO by removing the learned value function and estimating advantages through group-relative rewards, which substantially reduces memory and implementation complexity for long-context reasoning \citep{shao2024deepseekmath}. However, GRPO inherits a token-level clipping design: each token ratio is clipped separately, while the advantage is typically a response-level scalar shared by all tokens in the completion. This creates a mismatch between the token-level control variable and the sequence-level reward signal.

A natural response to this mismatch is to optimize at the sequence level. GSPO constructs a response-level power multiplier from the likelihood ratio of the whole generated sequence and performs sequence-level clipping, rewarding, and optimization \citep{zheng2025gspo}. This aligns the unit of policy correction with the unit of reward and avoids treating individual tokens as independently rewarded decisions. Yet sequence-level clipping introduces a new challenge: the sequence log-ratio is a sum over tokens, so its scale depends strongly on the response length. FSPO identifies this length fairness issue and clips the sequence log-importance ratio with a band that scales as \(\sqrt{L}\), where \(L\) is the output length \citep{mao2025clip}. This is an important step toward fair sequence-level RL, but the clipping width remains determined by fixed hyperparameters. In long-chain-of-thought RL, the distribution of log-ratios can vary substantially across training stages as the policy improves, the entropy changes, and the model discovers new reasoning strategies. A fixed clipping band may be too conservative during exploration, suppressing rare but valuable responses, and too loose during unstable phases, allowing harmful updates.

We therefore calibrate the clipping band from recent policy shifts to target a specified marginal tail probability. This lets the numerical width change with the score distribution. Conformal prediction motivates this approach: empirical quantiles define intervals with controlled miscoverage under exchangeability, and adaptive conformal methods address calibration under distribution shift \citep{vovk2005algorithmic,shafer2008tutorial,angelopoulos2023gentle,gibbs2021adaptive}. We use quantile calibration to select the sequence-level clipping band. Because training reuses data and changes the policy, our bounds explicitly retain CDF discrepancies and local drift. They are marginal over the response-length distribution; normalization by \(L\) alone does not imply calibration conditional on length.

We propose \textbf{C}onformal prediction-enabled \textbf{S}equence-level \textbf{P}olicy \textbf{O}ptimization (\textbf{CSPO}). Before each training step, CSPO uses a rolling buffer of retained training scores \(W=\exp(S/L)\), where \(S=\log[\pi_{\vtheta}(\vy\mid\vx)/\pi_{\vtheta_{\mathrm{old}}}(\vy\mid\vx)]\) is the sequence-level log-ratio and \(L=\operatorname{len}(\vy)\). Given a target miscoverage level \(\alpha\), CSPO estimates lower and upper empirical quantiles \((q_{\mathrm{lo}},q_{\mathrm{hi}})\) from completed blocks of retained scores, anchors the resulting interval at \(W=1\), and holds it fixed throughout the optimizer updates for that rollout batch. Scores computed on the training batch are retained for subsequent steps. The resulting sequence-level surrogate targets a marginal tail budget for fresh scores. Our contributions are threefold:
\begin{enumerate}
    \item We identify the need for \emph{adaptive} sequence-level clipping in RL post-training for reasoning language models, complementing prior work on sequence-level optimization and length-fair clipping.
    \item We propose CSPO, which clips GSPO power scores using rolling empirical quantiles of retained training scores to target a specified marginal tail budget.
    \item We bound tail probabilities for the mixture of scores within each training step in terms of \(\alpha\), CDF discrepancies, and local drift. Under additional moment conditions, these bounds control the objective and gradient error relative to the unclipped GSPO power surrogate. Empirically, CSPO improves Pass@1 over GSPO on difficult mathematical reasoning benchmarks and achieves the highest overall success rate on ALFWorld among the evaluated methods.
\end{enumerate}

\section{Methodology}
\label{sec:method}
A training step \(\tau\) consists of one rollout batch and \(E_\tau\) optimizer updates. The rollout parameter \(\vtheta_{\tau,0}\) remains fixed in every denominator; \(\vtheta_{\tau,k}\) is the current parameter before update \(k\), for \(k=0,\ldots,E_\tau-1\), and \(\vtheta_{\tau+1,0}=\vtheta_{\tau,E_\tau}\). We use \(t\) for the current training step and \(\tau<t\) for a retained step. At a fixed \((t,k)\), the shorter notation \(\vtheta=\vtheta_{t,k}\) and \(\vtheta_{\mathrm{old}}=\vtheta_{t,0}\) is used in the surrogate formulas. Let \(F_{t,k}\) be the population CDF of the score evaluated at this policy pair, as defined in \eqref{eq:forward_score_cdf}. Its band risk is
\begin{equation}
p_{t,k}(a,b)=F_{t,k}(a^-)+1-F_{t,k}(b).
\label{eq:intro_tail_risk}
\end{equation}
We first define the sequence score, then the rolling buffer used to select future clipping bands. The same training data supply both the policy update and the retained scores.

Let \(\vx\sim p\) be a prompt and \(\vy_i=(y_{i,1},\ldots,y_{i,L_i})\sim\pi_{\vtheta_{\rm old}}(\cdot\mid\vx)\) one of \(G\) responses with sequence reward \(\mathcal R_i\). We use the group-normalized advantage \(\widehat A_i=[\mathcal R_i-\operatorname{mean}_j(\mathcal R_j)]/[\operatorname{std}_j(\mathcal R_j)+\epsilon_A]\) and token ratio \(r_{i,\ell}(\vtheta)=\pi_{\vtheta}(y_{i,\ell}\mid\vh_{i,\ell})/\pi_{\vtheta_{\rm old}}(y_{i,\ell}\mid\vh_{i,\ell})\). Appendix~\ref{app:baseline_setup} records the full factorization and token-level GRPO reference objective. Sequence-level optimization uses a single multiplier for each complete response. For response \(i\), define the sequence log-ratio, the exact sequence likelihood ratio, and the GSPO power score by
\begin{equation}
\begin{aligned}
S_i(\vtheta)
&=
\log
\frac{\pi_{\vtheta}(\vy_i\mid\vx_i)}
{\pi_{\vtheta_{\mathrm{old}}}(\vy_i\mid\vx_i)}
=
\sum_{\ell=1}^{L_i}\log r_{i,\ell}(\vtheta),\ 
R_i(\vtheta) =\exp\!\left(S_i(\vtheta)\right),\ 
W_i(\vtheta)
=
R_i(\vtheta)^{1/L_i}
.
\end{aligned}
\label{eq:seq_log_ratio}
\end{equation}
Here \(R_i\) is the exact density ratio between sequence distributions when absolute continuity holds. The GSPO multiplier \(W_i\) is the geometric mean of the token ratios and generally differs from an exact importance weight.

In indexed form, the stored score is
\begin{equation}
S_{\tau,k}(\vx,\vy)=\log\frac{\pi_{\vtheta_{\tau,k}}(\vy\mid\vx)}{\pi_{\vtheta_{\tau,0}}(\vy\mid\vx)},\qquad
W_{\tau,k}(\vx,\vy)=\exp\!\left(\frac{S_{\tau,k}(\vx,\vy)}{L(\vy)}\right).
\label{eq:indexed_cached_score}
\end{equation}
This is the cached score for the length exponent \(\beta=1\) analyzed here. With a general exponent, the cached value is \(\exp(S_{\tau,k}/L^\beta)\). All score CDFs refer to this exponentiated, length-normalized value. Numerical capping inside the loss does not redefine the cached score.

With fixed constants \(\epsilon_{\mathrm{low}},\epsilon_{\mathrm{high}}>0\), the response-level GSPO objective can be written as
\begin{equation}
J_{\mathrm{seq}}(\vtheta)
=
\mathbb{E}_{\vx,\{\vy_i\}\sim\pi_{\vtheta_{\mathrm{old}}}}
\left[
\frac{1}{G}\sum_{i=1}^{G}
\min\!\left\{
W_i(\vtheta)\widehat A_i,\,
\operatorname{clip}\!\left(
W_i(\vtheta),
1-\epsilon_{\mathrm{low}},
1+\epsilon_{\mathrm{high}}
\right)\widehat A_i
\right\}
\right].
\label{eq:seq_objective}
\end{equation}
This fixed-band objective is the baseline that CSPO adapts. In particular, CSPO reduces exactly to this clipping rule whenever its calibrated endpoints are fixed at \(1-\epsilon_{\mathrm{low}}\) and \(1+\epsilon_{\mathrm{high}}\).


The per-update risk is \(p_{t,k}^{\mathrm{band}}:=p_{t,k}(a_t,b_t)\). The step-level risk \(p_t^{\mathrm{band}}\) will denote its mixture weighted by the numbers of scores actually retained at each update. It is a \emph{potential} clipping risk: the actual constant branch also depends on the advantage sign. The parameter \(\alpha=\alpha_{\mathrm{lo}}+\alpha_{\mathrm{hi}}\) specifies the nominal tail mass used to regularize the surrogate. It does not specify a KL radius or guarantee reward improvement.

The rollout batch \(\mathcal D_t\) is generated once from \(\pi_{\vtheta_{t,0}}\) and may be reused across optimizer updates. Let \(m_{\tau,k}\) be the number of score entries actually retained from forward evaluations at \(\vtheta_{\tau,k}\), and let \(\mathcal C_\tau=\{k:m_{\tau,k}>0\}\). A forward evaluation that is not collected contributes no entries. Repeatedly retained scores from the same response count as separate entries, each at the parameter used in that forward evaluation. Write
\[
\mathcal A_{\tau,k}=(W_{\tau,k,j})_{j=1}^{m_{\tau,k}},\qquad
\mathcal A_\tau=\biguplus_{k\in\mathcal C_\tau}\mathcal A_{\tau,k},\qquad
m_\tau=\sum_{k\in\mathcal C_\tau}m_{\tau,k}.
\]
Finalizing a block concatenates the recorded values; it does not recompute them at \(\vtheta_{\tau,E_\tau}\). Only completed nonempty blocks are included in the rolling buffer.

Let \(\mathcal I_t\subseteq\{1,\ldots,t-1\}\) index the completed nonempty blocks retained before step \(t\), and flatten each \(\mathcal A_\tau\) as \((W_{\tau,j})_{j=1}^{m_\tau}\). Thus
\[
\mathcal{A}_{\tau}=(W_{\tau,j})_{j=1}^{m_\tau},
\qquad
\mathcal{B}_t=\biguplus_{\tau\in\mathcal{I}_t}\mathcal{A}_{\tau},
\qquad
n_t=\sum_{\tau\in\mathcal{I}_t}m_\tau,
\]
with empirical CDF
\(
\widehat F_{\mathcal{B}_t}(w)
=
\frac{1}{n_t}
\sum_{\tau\in\mathcal{I}_t}
\sum_{j=1}^{m_\tau}
\mathbf{1}\{W_{\tau,j}\le w\}.
\)
We use the inverse empirical CDF to define the quantiles:
\(
Q_p(\mathcal{B}_t)
=
\inf\{w:\widehat F_{\mathcal{B}_t}(w)\ge p\}, p\in(0,1).
\)
For \(\alpha_{\mathrm{lo}},\alpha_{\mathrm{hi}}>0\) with \(\alpha_{\mathrm{lo}}+\alpha_{\mathrm{hi}}=\alpha<1\), the raw endpoints are
\begin{equation}
q_{\mathrm{lo},t}=Q_{\alpha_{\mathrm{lo}}}(\mathcal{B}_t),
\qquad
q_{\mathrm{hi},t}=Q_{1-\alpha_{\mathrm{hi}}}(\mathcal{B}_t).
\label{eq:cspo_quantiles}
\end{equation}
A symmetric allocation uses \(\alpha_{\mathrm{lo}}=\alpha_{\mathrm{hi}}=\alpha/2\); asymmetric allocations are allowed but must be fixed before selecting the clipping quantiles.

The unrestricted empirical quantiles above admit the blockwise result in Theorem~\ref{thm:drift}. To express an optional pooled certificate on a finite set of thresholds, grid mode predeclares a finite set \(\mathcal G_M=\{g_1<\cdots<g_M\}\subset(0,\infty)\) before training. Define outward rounding by
\(
\lfloor u\rfloor_{\mathcal G}
=
\max\bigl(\{0\}\cup\{g\in\mathcal G_M:g\le u\}\bigr),
\lceil u\rceil_{\mathcal G}
=
\min\bigl(\{+\infty\}\cup\{g\in\mathcal G_M:g\ge u\}\bigr),
\)
and set
\begin{equation}
\widetilde q_{\mathrm{lo},t}
=
\lfloor q_{\mathrm{lo},t}\rfloor_{\mathcal G},
\qquad
\widetilde q_{\mathrm{hi},t}
=
\lceil q_{\mathrm{hi},t}\rceil_{\mathcal G}.
\label{eq:grid_rounded_quantiles}
\end{equation}
Outward rounding can only enlarge the interval. A logarithmically spaced grid containing \(1\) makes the approximation multiplicative; reaching \(0\) or \(+\infty\) disables clipping on that side and must be reported. Because each rollout step begins at \(W_{t,0}=1\), the operational endpoints are
\begin{equation}
a_t=\min\{\widetilde q_{\mathrm{lo},t},1\},
\qquad
b_t=\max\{\widetilde q_{\mathrm{hi},t},1\}.
\label{eq:cspo_operational_bounds}
\end{equation}
When grid mode is disabled, set \(\widetilde q_{\mathrm{lo},t}=q_{\mathrm{lo},t}\) and \(\widetilde q_{\mathrm{hi},t}=q_{\mathrm{hi},t}\). If the buffer is empty, use warm-start bounds \((1-c,1+c)\), \(c\in(0,1)\). The CSPO surrogate is
\begin{equation}
J_{\mathrm{cspo}}(\vtheta)
=
\mathbb{E}_{\vx,\{\vy_i\}\sim\pi_{\vtheta_{t,0}}}
\left[
\frac{1}{G}\sum_{i=1}^{G}
\min\!\left\{
W_i(\vtheta)\widehat A_i,
\operatorname{clip}(W_i(\vtheta),a_t,b_t)\widehat A_i
\right\}
\right],
\label{eq:cspo_objective}
\end{equation}
where \((a_t,b_t)\) is fixed throughout the optimizer updates of step \(t\). The empirical objective is evaluated at \(\vtheta_{t,k}\) on the responses used by update \(k\). Algorithm~\ref{alg:cspo} retains the scores selected by the existing collection schedule and appends their concatenation for subsequent steps.

\begin{algorithm}[htbp]
\caption{CSPO with a rolling buffer of training scores and grid mode}
\label{alg:cspo}
\begin{algorithmic}[1]
\Require tail levels \(\alpha_{\mathrm{lo}},\alpha_{\mathrm{hi}}\), horizon \(K\), predeclared grid \(\mathcal G_M\), warm-start radius \(c\in(0,1)\)
\State Initialize block buffer \(\mathcal{B}\gets\varnothing\)
\For{training step \(t=1,2,\ldots\)}
    \State Freeze the rollout parameter \(\vtheta_{\mathrm{old}}\gets\vtheta_{t,0}\)
    \State If \(\mathcal B=\varnothing\), use \((a_t,b_t)=(1-c,1+c)\); otherwise compute raw quantiles by \eqref{eq:cspo_quantiles}, round outward by \eqref{eq:grid_rounded_quantiles}, and anchor by \eqref{eq:cspo_operational_bounds}
    \State Generate the rollout batch \(\mathcal D_t\) from \(\pi_{\vtheta_{t,0}}\) and compute its group-relative advantages
    \For{optimizer update \(k=0,\ldots,E_t-1\)}
        \State Evaluate the loss and scores \(W_{t,k}\) on the responses used by this update, with numerator \(\pi_{\vtheta_{t,k}}\) and denominator \(\pi_{\vtheta_{t,0}}\)
        \State If collection is enabled, retain the corresponding \(m_{t,k}\) score entries in \(\mathcal A_{t,k}\)
        \State Apply the optimizer update to obtain \(\vtheta_{t,k+1}\)
    \EndFor
    \State Finalize \(\mathcal A_t=\biguplus_{k\in\mathcal C_t}\mathcal A_{t,k}\) at its recorded values and set \(\vtheta_{t+1,0}\gets\vtheta_{t,E_t}\)
    \State Record empirical diagnostics; append the nonempty \(\mathcal A_t\) for subsequent steps and retain the most recent \(K\) blocks
\EndFor
\end{algorithmic}
\end{algorithm}

\section{Theoretical Results}
\label{sec:theory}

The analysis first bounds tail probabilities for the mixture of scores collected within a training step, accounting for CDF discrepancies and drift. Moment conditions then translate these bounds into objective and gradient error relative to the GSPO power surrogate. Appendix~\ref{app:theory} gives the results for length groups, quantile tracking, exchangeable calibration, and error accumulation across updates under their additional assumptions.

Fix the realized optimization path. For each \((\tau,k)\), define the population CDF by pushing the rollout law through the score map in \eqref{eq:indexed_cached_score}:
\begin{equation}
F_{\tau,k}(w)=\int\mathbf1\{W_{\tau,k}(\vx,\vy)\le w\}\,p(d\vx)\pi_{\vtheta_{\tau,0}}(d\vy\mid\vx).
\label{eq:forward_score_cdf}
\end{equation}
The parameters in this integral are held at their realized values, including any dependence on \(\mathcal D_\tau\). For \(m_\tau>0\), set \(\nu_{\tau,k}=m_{\tau,k}/m_\tau\) and define
\begin{equation}
F_\tau=\sum_{k\in\mathcal C_\tau}\nu_{\tau,k}F_{\tau,k},\qquad
\widehat F_\tau=\sum_{k\in\mathcal C_\tau}\nu_{\tau,k}\widehat F_{\tau,k},\qquad
\widehat F_{\tau,k}(w)=\frac1{m_{\tau,k}}\sum_{j=1}^{m_{\tau,k}}\mathbf1\{W_{\tau,k,j}\le w\}.
\label{eq:step_score_mixture}
\end{equation}
Thus \(F_\tau\) is the population mixture associated with the actual collection weights, not the score law of a single policy pair. The weights count entries rather than distinct responses. In particular, repeated scoring does not create additional independent rollouts, and the entries in \(\mathcal A_\tau\) are not assumed i.i.d.

Let \(\mathcal H_\tau\) denote the history after step \(\tau\), including its parameter path, collection counts, and stored scores. Conditional on this history, a population score \(W_\tau^\star\) is defined by selecting \(k\) with probabilities \(\nu_{\tau,k}\) and evaluating \(W_{\tau,k}\) on an independent hypothetical pair \((\vx_\tau^\star,\vy_\tau^\star)\) from \(p\pi_{\vtheta_{\tau,0}}\), with \(L_\tau^\star=L(\vy_\tau^\star)\). This defines \(F_\tau\) and its risk without an additional algorithmic sampling step:
\begin{equation}
p_\tau(a,b)=F_\tau(a^-)+1-F_\tau(b)
=\sum_{k\in\mathcal C_\tau}\nu_{\tau,k}p_{\tau,k}(a,b),\qquad
p_\tau^{\mathrm{band}}=p_\tau(a_\tau,b_\tau).
\label{eq:step_mixture_risk}
\end{equation}
For the retained block weights \(\lambda_{\tau,t}=m_\tau/n_t\), define the mixture CDF and its current-round discrepancy by
\begin{equation}
\overline F_t(w)
=
\sum_{\tau\in\mathcal I_t}\lambda_{\tau,t}F_\tau(w),
\qquad
D_t
=
\|\overline F_t-F_t\|_\infty.
\label{eq:weighted_drift}
\end{equation}
The weighted drift \(D_t\) is the quantity needed by the quantile argument and obeys \(D_t\le\Delta_K\) under Assumption~\ref{assump:local_stability}; retaining it avoids replacing an average discrepancy by the worst block unnecessarily.

To separate the empirical quantile argument from dependence within a reused rollout batch, define the realized CDF discrepancies
\begin{equation}
\xi_{\tau,k}=\|\widehat F_{\tau,k}-F_{\tau,k}\|_\infty,\qquad
\xi_\tau=\|\widehat F_\tau-F_\tau\|_\infty
\le\sum_{k\in\mathcal C_\tau}\nu_{\tau,k}\xi_{\tau,k}.
\label{eq:realized_cdf_errors}
\end{equation}
These are quantities defined on the realized training path, not assumed sampling rates. The bounds below retain them explicitly; the number of cached entries alone does not bound them under adaptive reuse.

The following condition bounds drift between the step-level mixture CDFs.

\begin{assumption}
\label{assump:local_stability}
For a step \(t\) with \(m_t>0\), let \(\mathcal I_t\) contain the latest \(K_t=|\mathcal I_t|\) retained blocks. Almost surely along the training trajectory,
\begin{equation}
\sup_{w\in\mathbb{R}}
|F_t(w)-F_\tau(w)|
\le
\Delta_K,
\qquad
\forall \tau\in\mathcal{I}_t,
\label{eq:window_drift}
\end{equation}
where \(\Delta_K\in[0,1]\) is the maximum local CDF drift over the retained window.
\end{assumption}

Theorem~\ref{thm:drift} gives the pathwise, tie-robust bound \(p_t^{\mathrm{band}}\le\alpha+2\varepsilon_t+2D_t\), where \(\varepsilon_t=\sum_{\tau\in\mathcal I_t}(m_\tau/n_t)\xi_\tau\). The risk concerns the current step mixture \(F_t\). Corollary~\ref{cor:drift_atoms} quantifies conservatism through its largest endpoint atom. Theorem~\ref{thm:drift} is marginal over response length; Proposition~\ref{prop:length_group_fairness} and Corollary~\ref{cor:length_group_control} state the additional discrepancy required for length-stratified transfer.

\subsection{Optimization consequences of calibrated clipping}
\label{sec:optimization_consequences}

We next connect the calibrated tail event to the optimization objective actually used by CSPO. Fix an update \((t,k)\), hold \(\vtheta=\vtheta_{t,k}\) and \(\vtheta_{\mathrm{old}}=\vtheta_{t,0}\) at their realized values, and fix finite thresholds \(0<a\le1\le b<\infty\). Rounds in which grid rounding returns \(a=0\) or \(b=+\infty\) simply disable clipping on that side and are excluded from the finite-endpoint distortion statements below. Draw
\(
\vx\sim p(\vx),
\vy_{1:G}\stackrel{\mathrm{i.i.d.}}{\sim}
\pi_{\vtheta_{\mathrm{old}}}(\cdot\mid\vx),
U\sim\operatorname{Unif}\{1,\ldots,G\},
\)
where \(U\) is conditionally independent of the response group. For each response, define
\(
R_{\vtheta,i}
=
\frac{\pi_{\vtheta}(\vy_i\mid\vx)}{\pi_{\vtheta_{\mathrm{old}}}(\vy_i\mid\vx)},
W_{\vtheta,i}=R_{\vtheta,i}^{1/L_i},
\)
compute the group-relative advantages \((\widehat A_1,\ldots,\widehat A_G)\), and set
\[
(\vy,\widehat A,L,R_{\vtheta},W_{\vtheta})
=
(\vy_U,\widehat A_U,L_U,R_{\vtheta,U},W_{\vtheta,U}).
\]
All expectations below are taken over this extended probability space; equivalently, they average \(G^{-1}\sum_i\) within a random prompt group. The marginal CDF of this fresh score is \(F_{t,k}\); the step-level \(F_t\) additionally averages over collected update indices. Assume \(\pi_{\vtheta}\) is absolutely continuous with respect to \(\pi_{\vtheta_{\mathrm{old}}}\) and \(L\ge L_{\min}\ge2\) almost surely. Here \(R_{\vtheta}\) is the likelihood ratio of the complete sampled object, including the declared EOS, termination, and truncation convention. The analysis requires \(L\ge2\), which can be enforced by a documented effective-length rule; otherwise Lemma~\ref{lem:power_ratio_second_moment} must be replaced by an explicit moment assumption. Before bounding clipping distortion, we need a moment guarantee that does not itself assume a bounded likelihood ratio. The power \(1/L\) supplies a uniform second-moment bound for the GSPO multiplier.

\begin{lemma}
\label{lem:power_ratio_second_moment}
Under the preceding conditions,
\begin{equation}
\mathbb{E}_{\vtheta_{\mathrm{old}}}[W_{\vtheta}^{2}]
\le
1+\frac{2}{L_{\min}}.
\label{eq:power_ratio_second_moment}
\end{equation}
\end{lemma}

To translate this moment control into objective fidelity, define
\begin{equation}
\ell_{a,b}(w,A)
=
\min\{wA,\operatorname{clip}(w,a,b)A\},\ 
\mathcal{J}_{\mathrm{G}}(\vtheta)
=
\mathbb{E}_{\vtheta_{\mathrm{old}}}
[W_{\vtheta}\widehat A],
\ 
\mathcal{J}_{\mathrm{C}}(\vtheta)
=
\mathbb{E}_{\vtheta_{\mathrm{old}}}
[\ell_{a,b}(W_{\vtheta},\widehat A)],
\label{eq:power_surrogates}
\end{equation}
and assume that the fixed update advantage obeys \(|\widehat A|\le A_{\max}\) almost surely. This bound is operational for standard group normalization: with a population-standard-deviation convention one may take \(A_{\max}=\sqrt G\), and with an unbiased sample-standard-deviation convention \(A_{\max}=\sqrt{G-1}\); adding \(\epsilon_A>0\) only decreases the magnitude. Implementations using a different normalization must state the corresponding constant. Let
\(
\widehat A_+=\max\{\widehat A,0\},\ 
\widehat A_-=\max\{-\widehat A,0\},\ 
p_{\vtheta}^{+}=\mathbb{P}(W_{\vtheta}>b),\ 
p_{\vtheta}^{-}=\mathbb{P}(W_{\vtheta}<a),\ 
p_{\vtheta}=p_{\vtheta}^{+}+p_{\vtheta}^{-}.
\)
The next proposition bounds the decrease in the surrogate objective caused by clipping.

\begin{proposition}
\label{prop:objective_distortion}
For every \(\vtheta\) satisfying the preceding conditions,
\begin{align}
0
&\le
\mathcal{J}_{\mathrm{G}}(\vtheta)
-
\mathcal{J}_{\mathrm{C}}(\vtheta)
=
\mathbb{E}_{\vtheta_{\mathrm{old}}}
\!\left[
\widehat A_+(W_{\vtheta}-b)_+
+
\widehat A_-(a-W_{\vtheta})_+
\right]
\label{eq:exact_objective_distortion}\\
&\le
A_{\max}
\left[
\sqrt{
\left(1+\frac{2}{L_{\min}}\right)
p_{\vtheta}^{+}
}
+
a p_{\vtheta}^{-}
\right]
\label{eq:objective_distortion_tailwise}\\
&\le
A_{\max}
\left[
\sqrt{
\left(1+\frac{2}{L_{\min}}\right)
p_{\vtheta}
}
+
a p_{\vtheta}
\right].
\label{eq:objective_distortion_total}
\end{align}
\end{proposition}

Objective fidelity does not by itself show that the clipped update points in a useful direction. Define the normalized sequence score vector
\(
\psi_{\vtheta}(\vx,\vy)
=
\frac{1}{L}
\nabla_{\vtheta}\log\pi_{\vtheta}(\vy\mid\vx)
\)
and assume \(\|\psi_{\vtheta}\|\le M_{\psi}\) almost surely, together with the regularity needed to differentiate the unclipped expectation. At \(W_{\vtheta}=a\) or \(W_{\vtheta}=b\), the per-sample clipped objective can be nondifferentiable. We define the CSPO update rule to select the derivative of the unclipped branch at equality, and let \(g_{\mathrm{C}}\) denote the expectation of this measurable selected direction. Away from the endpoints it equals the classical clipped-objective gradient. An implementation invoking Proposition~\ref{prop:gradient_fidelity} must use the same endpoint derivative convention; otherwise an additional boundary term is needed. Under this convention, the actual saturation event is
\begin{equation}
\mathcal{C}_{\vtheta}
=
\{\widehat A>0,W_{\vtheta}>b\}
\cup
\{\widehat A<0,W_{\vtheta}<a\}.
\label{eq:actual_saturation_event}
\end{equation}
Thus \(\mathcal{C}_{\vtheta}\subseteq \{W_{\vtheta}<a\ \mathrm{or}\ W_{\vtheta}>b\}\); positive-advantage samples saturate only in the upper tail, and negative-advantage samples only in the lower tail.

\begin{proposition}
\label{prop:gradient_fidelity}
Under the preceding conditions,
\begin{equation}
g_{\mathrm{G}}
:=
\nabla\mathcal{J}_{\mathrm{G}}(\vtheta)
=
\mathbb{E}_{\vtheta_{\mathrm{old}}}
[W_{\vtheta}\widehat A\psi_{\vtheta}],
\qquad
g_{\mathrm{C}}
:=
\mathbb{E}_{\vtheta_{\mathrm{old}}}
[W_{\vtheta}\widehat A\psi_{\vtheta}
\mathbf{1}_{\mathcal{C}_{\vtheta}^{c}}],
\label{eq:power_gradients}
\end{equation}
and
\begin{equation}
\|g_{\mathrm{C}}-g_{\mathrm{G}}\|
\le
A_{\max}M_{\psi}
\sqrt{
\left(1+\frac{2}{L_{\min}}\right)p_{\vtheta}
}.
\label{eq:gradient_fidelity_bound}
\end{equation}
Moreover, the per-sample clipped gradient \(X_{\mathrm{C}} = W_{\vtheta}\widehat A\psi_{\vtheta} \mathbf{1}_{\mathcal{C}_{\vtheta}^{c}}\) satisfies
\begin{equation}
\mathbb{E}_{\vtheta_{\mathrm{old}}}
[\|X_{\mathrm{C}}\|^2]
\le
A_{\max}^{2}M_{\psi}^{2}
\left(1+\frac{2}{L_{\min}}\right).
\label{eq:clipped_gradient_second_moment}
\end{equation}
\end{proposition}

A sharper diagnostic, which also requires gradient information, is the expected gradient magnitude removed by clipping
\begin{equation}
\tau_{\vtheta}
=
\mathbb E\!\left[
W_{\vtheta}|\widehat A|\,\|\psi_{\vtheta}\|
\mathbf 1_{\mathcal C_{\vtheta}}
\right],
\qquad
\|g_{\mathrm G}-g_{\mathrm C}\|\le\tau_{\vtheta}.
\label{eq:saturated_gradient_mass}
\end{equation}
This inequality follows directly from \eqref{eq:power_gradients}. Calibrating \(\tau_{\vtheta}\) would require advantages and gradient norms in addition to score quantiles; the current rule does not control it at level \(\alpha\). The bound on the original normalized score gradient is an assumption that need not hold for a neural network. Clipping individual gradients changes the gradient field and requires a separate bias analysis; clipping the batch gradient does not imply this assumption. Proposition~\ref{prop:gradient_fidelity} therefore requires the stated \(M_\psi\) condition or a replacement based on finite moments.

With \(g_{\mathrm{G}}\) and \(g_{\mathrm{C}}\) defined in Proposition~\ref{prop:gradient_fidelity}, let
\(
B_{\vtheta}
=
A_{\max}M_{\psi}
\sqrt{
\left(1+\frac{2}{L_{\min}}\right)p_{\vtheta}
}.
\)
Proposition~\ref{prop:gradient_fidelity} implies \(\langle g_{\mathrm{G}},g_{\mathrm{C}}\rangle \ge \|g_{\mathrm{G}}\|(\|g_{\mathrm{G}}\|-B_{\vtheta})\). Hence the clipped gradient is an ascent direction for the unclipped power surrogate whenever \(B_{\vtheta}<\|g_{\mathrm{G}}\|\). For a step with \(m_t,n_t>0\) and an update \((t,k)\), the risk entering these fidelity bounds is \(p_{t,k}^{\mathrm{band}}\), not its step average. Define
\begin{equation}
D_{t,k}=\|\overline F_t-F_{t,k}\|_\infty,\qquad
c_{t,k}=\|F_{t,k}-F_t\|_\infty,\qquad D_{t,k}\le D_t+c_{t,k}.
\label{eq:iterate_drift}
\end{equation}
The same two-endpoint argument as in Theorem~\ref{thm:drift} gives, for every internal update,
\begin{equation}
p_{t,k}^{\mathrm{band}}\le\overline p_{t,k}
:=\min\{1,\alpha+2\varepsilon_t+2D_{t,k}\}
\le\min\{1,\alpha+2\varepsilon_t+2D_t+2c_{t,k}\}.
\label{eq:iterate_tail_control}
\end{equation}
Indeed, \(\|\widehat F_{\mathcal B_t}-F_{t,k}\|_\infty\le\varepsilon_t+D_{t,k}\), so each strict tail contributes at most its empirical tail budget plus this discrepancy. Thus \(\overline p_{t,k}\) may replace \(p_{\vtheta}\) in the objective and gradient bounds at \(\vtheta=\vtheta_{t,k}\). Control of an entire internal path uses the corresponding bounds for all its iterates.

Corollary~\ref{cor:one_step_power_improvement} gives the smoothness implication for a subsequent update at a parameter with a valid tail bound. Applying it along a path requires a bound at every iterate; Proposition~\ref{prop:path_deviation} states the additional condition for accumulating the gradient error. The following remark specifies the scope of these results.

\begin{remark}
\label{rem:power_surrogate_scope}
The optimization results compare CSPO with the unclipped GSPO power surrogate. In general, \(W_{\vtheta}=R_{\vtheta}^{1/L}\) is not a density ratio and
\[
\mathbb{E}_{\vtheta_{\mathrm{old}}}
[W_{\vtheta}f(\vx,\vy)]
\ne
\mathbb{E}_{\vtheta}[f(\vx,\vy)].
\]
Consequently, these results do not imply a KL trust-region constraint or monotonic improvement of the environment reward. Their content is that calibrated tail mass controls how far clipping moves the objective and gradient away from the corresponding GSPO surrogate.
\end{remark}

\subsection{Pooled bounds and quantile tracking}
\label{sec:pooled_tracking}
On the predeclared grid, only the following realized CDF discrepancy is needed:
\begin{equation}
r_t^{\mathcal G}=\max_{g\in\mathcal G_M}\max\!\left\{
|\widehat F_{\mathcal B_t}(g)-\overline F_t(g)|,
|\widehat F_{\mathcal B_t}(g^-)-\overline F_t(g^-)|\right\}.
\label{eq:grid_cdf_error}
\end{equation}
At the finite grid, define
\begin{equation}
D_{t,\mathcal G}
=
\max_{g\in\mathcal G_M}
\max\!\left\{
|\overline F_t(g)-F_t(g)|,
|\overline F_t(g^-)-F_t(g^-)|
\right\}.
\label{eq:grid_drift_radius}
\end{equation}
Here \(D_{t,\mathcal G}\le D_t\le\Delta_K\), and \(r_t^{\mathcal G}\le\varepsilon_t\) by the weighted triangle inequality.

\begin{theorem}
\label{thm:grid_pooled}
For a step \(t\) with \(m_t,n_t>0\), let \(r_t^{\mathcal G}\) and \(D_{t,\mathcal G}\) be defined above. Use the outward-rounded and anchored endpoints in \eqref{eq:grid_rounded_quantiles}--\eqref{eq:cspo_operational_bounds}. Then, on every realized training path,
\begin{equation}
p_t^{\mathrm{band}}
\le
\min\!\left\{
1,
\alpha+2r_t^{\mathcal G}+2D_{t,\mathcal G}
\right\}.
\label{eq:grid_pooled_control}
\end{equation}
\end{theorem}

This bound uses the actual pooled score discrepancy and allows dependent, repeatedly scored responses. Corollary~\ref{cor:window_tradeoff} separates its window dependence from linear local drift.
Theorem~\ref{thm:exchangeable_rank} in Appendix~\ref{app:theory} gives finite-sample rank control for exchangeable scores evaluated at a fixed policy pair. This assumption does not follow for the rolling buffer. Corollary~\ref{cor:observable_drift} relates the empirical tail and drift diagnostics to the population discrepancy, and Corollary~\ref{cor:anytime_grid} records the validity of \eqref{eq:grid_pooled_control} across steps. Proposition~\ref{prop:dynamic_tracking} further bounds how closely the endpoints follow population quantiles under the local inverse-CDF condition in Assumption~\ref{assump:quantile_identifiability}.

\section{Experiments}
\label{sec:experiments}

We evaluate CSPO on two settings: multi-step embodied decision making using ALFWorld \citep{shridhar2021alfworld}, and difficult mathematical reasoning using DAPO-Math-17k \citep{yu2025dapo}. 
ALFWorld is a text-based embodied benchmark in which the agent completes household goals through multi-step interaction. ALFWorld consists of six goal-conditioned task types: Pick \& Place (Pick), Examine in Light (Look), Clean \& Place (Clean), Heat \& Place (Heat), Cool \& Place (Cool), and Pick Two \& Place (Pick2).
For mathematical reasoning, we train on DAPO-Math-17k and evaluate on AIME \citep{maa_aime_2024,maa_aime_2025}, AMC \citep{maa_amc12_2023}, MATH500 \citep{lightman2023lets}, Minerva Math \citep{lewkowycz2022solving}, and OlympiadBench \citep{he2024olympiadbench}.
Both settings provide verifiable rewards, obtained directly from the environment for ALFWorld and from answer checking for mathematical reasoning. 

We use Qwen3-1.7B for ALFWorld and Qwen3-1.7B-Base for mathematical reasoning \citep{yang2025qwen3}. We compare CSPO with several policy optimization baselines: GRPO \citep{shao2024deepseekmath}, DAPO \citep{yu2025dapo}, and GSPO \citep{zheng2025gspo}. We generate \(G=8\) responses per prompt for mathematical reasoning and \(G=8\) trajectories per task for ALFWorld. Detailed training settings are provided in Appendix~\ref{app:experimental_setup}.

\subsection{Main results}
\label{subsec:main_results}

\begin{table}[htbp]
\centering
\caption{Performance on ALFWorld.}
\label{tab:alfworld_results}
\begin{tabular}{lccccccc}
\toprule
\multirow{2}{*}{Method} 
& \multicolumn{7}{c}{\textbf{ALFWorld}} \\
& Pick & Look & Clean & Heat & Cool & Pick2 & All \\
\midrule
Qwen3-1.7B 
& 0.00
& 40.00
& 3.23
& 0.00
& 0.00
& 0.00
& 3.91 \\

GSPO 
& 58.06	
& 66.67	
& 45.00	
& 61.54	
& 22.22	
& 20.83
& 31.25 \\

GRPO 
& 67.86	
& 66.67	 
& 56.52	 
& 61.54	
& 55.17	 
& \textbf{45.45} 
& 52.34 \\


DAPO 
& 76.92 
& \textbf{70.00}
& \textbf{72.00} 
& \textbf{75.00} 
& 56.52 
& 33.33 
& 53.12 \\

CSPO
& \textbf{93.33}
& \textbf{70.00}
& 68.00
& \textbf{75.00}
& \textbf{60.00}
& 44.00
& \textbf{64.06} \\
\bottomrule
\end{tabular}
\end{table}

\begin{table}[htbp]
\centering
\caption{Pass@1 performance on difficult mathematical reasoning benchmarks.}
\label{tab:hard_level_math}
\begin{tabular}{lccccccc}
\toprule
Method & AIME24 & AIME25 & AMC & MATH500 & Minerva & Oly. & Average \\
\midrule
Qwen3-1.7B-Base & 5.42 & 3.54 & 25.38 & 54.43 & 13.72 & 22.47 & 20.83 \\
GRPO & 11.40 & 5.65 & 39.43 & 70.49 & 21.90 & 33.27 & 30.36 \\
DAPO & 11.05 & 5.72 & 41.56 & 73.04 & 22.82 & 34.32 & 31.42 \\
GSPO & 15.08 & 8.33 & 42.86 & 73.36 & 22.68 & 37.83 & 33.36 \\
CSPO & \textbf{16.87} & \textbf{8.57} & \textbf{43.36} & \textbf{74.22} & \textbf{23.19} & \textbf{38.44} & \textbf{34.11} \\
\bottomrule
\end{tabular}
\end{table}

Table~\ref{tab:alfworld_results} reports performance on ALFWorld. CSPO achieves the best overall score of \(64.06\), improving over DAPO by \(10.94\) absolute points. CSPO achieves the best or tied-best performance on four of the six task types, including Pick, Look, Heat, and Cool. These results suggest that CSPO is effective for multi-step interactive tasks.

Table~\ref{tab:hard_level_math} reports Pass@1 performance on difficult mathematical reasoning benchmarks. CSPO achieves the best average score of \(34.11\), improving over GSPO by \(0.75\) absolute points. Compared with the Qwen3-1.7B-Base model, CSPO improves the average score by \(13.28\) absolute points. The gains are consistent across AIME, AMC, MATH500, Minerva, and Olympiad-level evaluations, suggesting that adaptive conformal sequence-level clipping is beneficial for difficult reasoning tasks.

\subsection{Training dynamics}

Figure~\ref{fig:learning_curves_alfworld} shows training rewards, validation success rates, and the calibrated endpoints on ALFWorld. The endpoints fluctuate throughout training as the distribution of retained scores changes. The upper and lower endpoints follow different trajectories, allowing the interval to adapt separately to the two tails.

Figure~\ref{fig:learning_curves_dapo_math} shows the corresponding curves for DAPO-Math-17k, including Pass@1 on AIME24. During early training, the upper endpoint decreases toward one and the lower endpoint increases toward one; both then vary within a narrower range. The band thus contracts as the retained power scores become more concentrated. These curves describe the behavior of the adaptive band; assessing its population tail probability also requires the CDF discrepancies in Section~\ref{sec:theory}.

\begin{figure*}[htbp]
\centering
\includegraphics[width=0.6\textwidth]{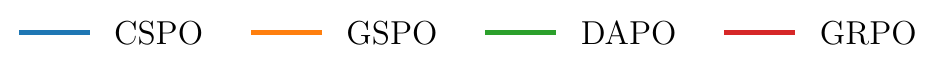}
\subfloat[Reward curves.]{\centering\includegraphics[trim=0cm 0.5cm 0cm 0cm,width=0.245\textwidth]{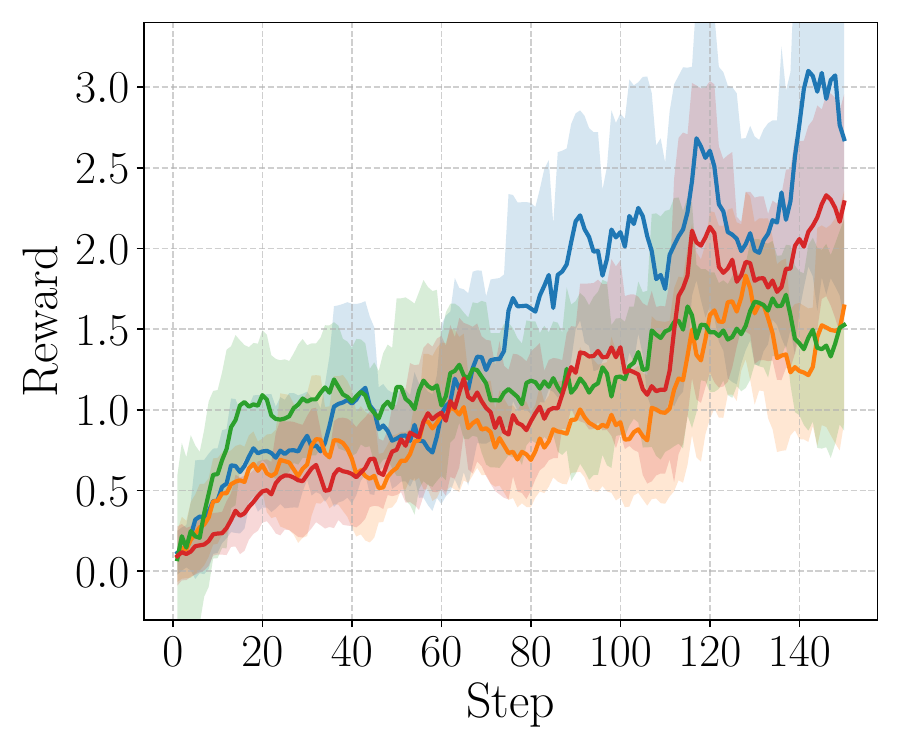}}
\subfloat[Success rate.]{\centering\includegraphics[trim=0cm 0.5cm 0cm 0cm,width=0.245\textwidth]{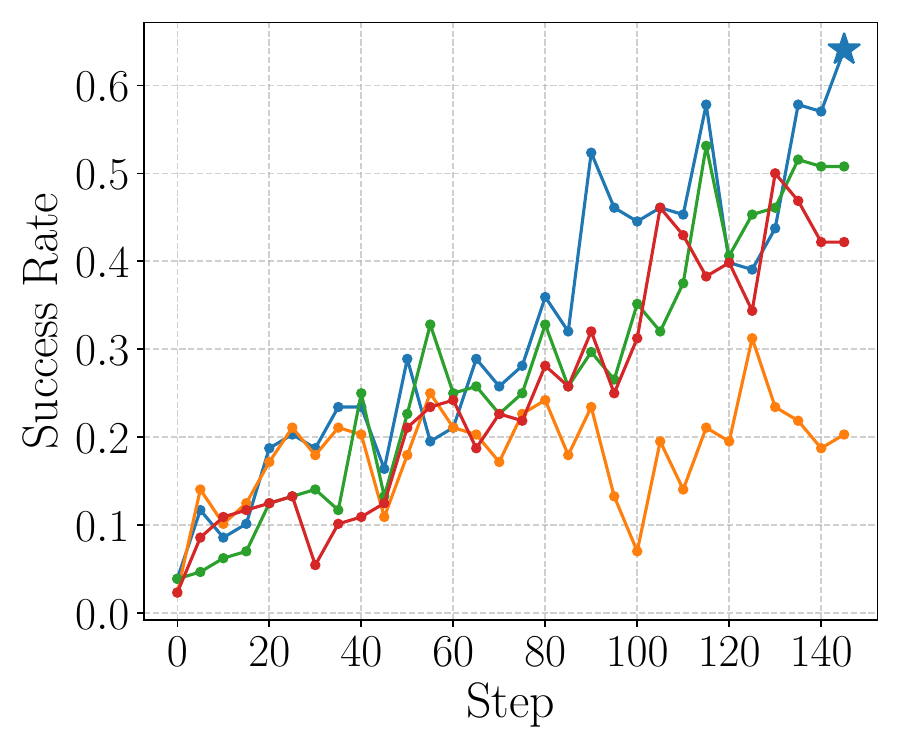}}
\subfloat[Dynamics of \(q_{\mathrm{hi}}\).]{\centering\includegraphics[trim=0cm 0.5cm 0cm 0cm,width=0.245\textwidth]{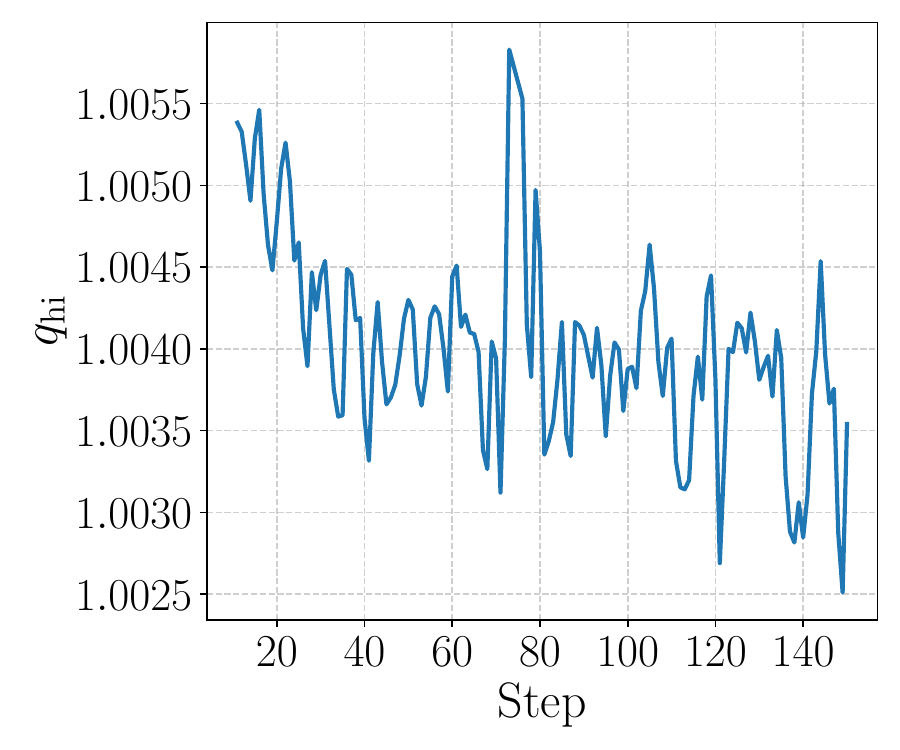}}
\subfloat[Dynamics of \(q_{\mathrm{lo}}\).]{\centering\includegraphics[trim=0cm 0.5cm 0cm 0cm,width=0.245\textwidth]{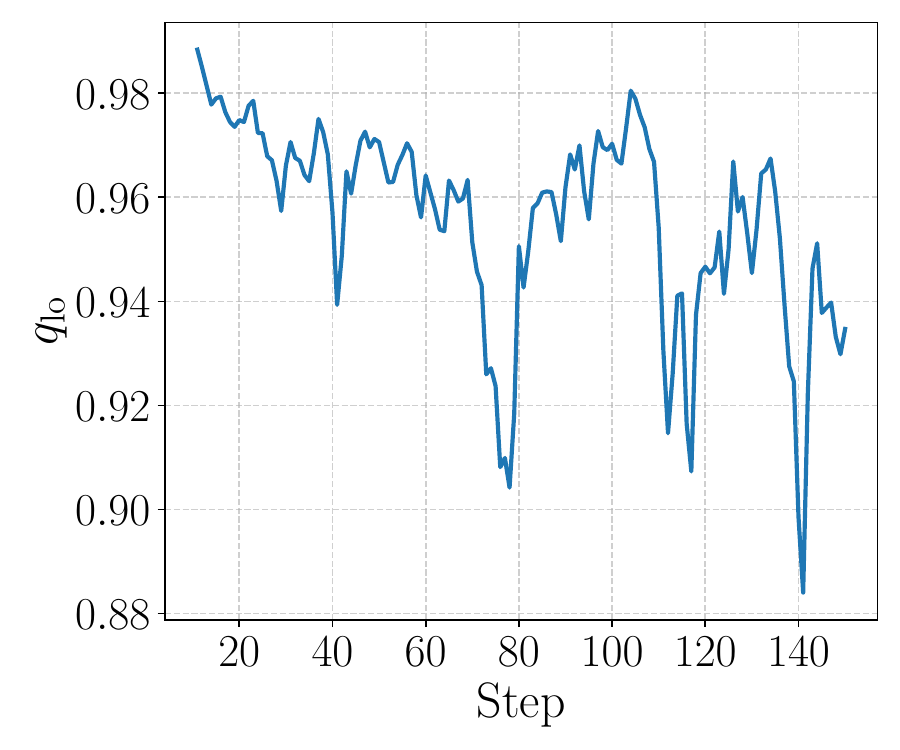}}
\caption{Learning curves and evaluation results on ALFWorld.}
\label{fig:learning_curves_alfworld}
\end{figure*}

\begin{figure*}[htbp]
\centering
\subfloat[Reward curves.]{\centering\includegraphics[trim=0cm 0.5cm 0cm 0cm,width=0.245\textwidth]{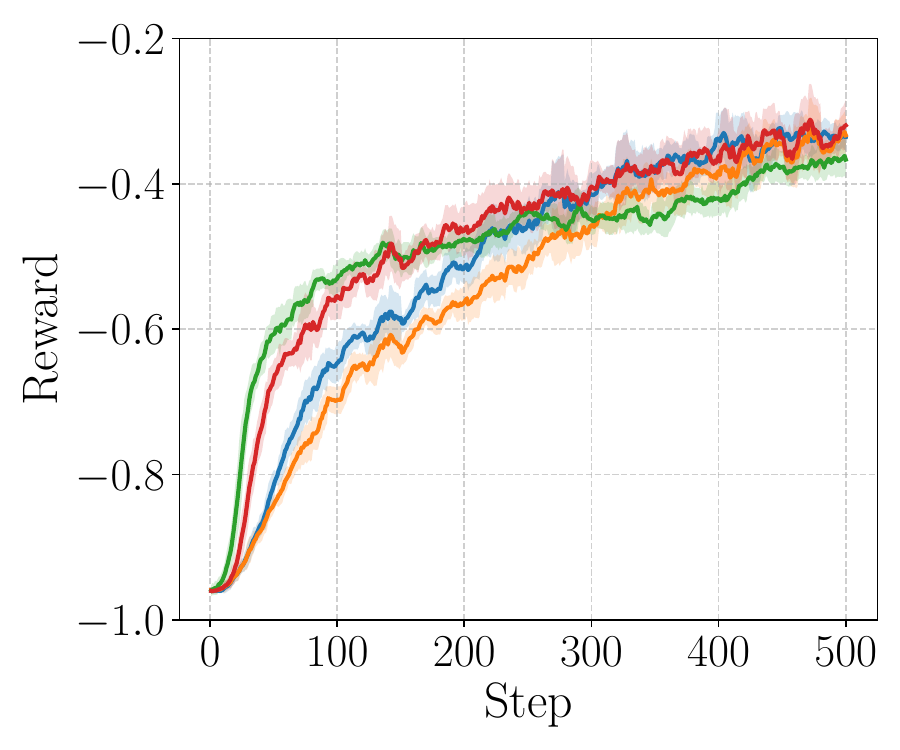}}
\subfloat[Pass@1 on AIME24.]{\centering\includegraphics[trim=0cm 0.5cm 0cm 0cm,width=0.245\textwidth]{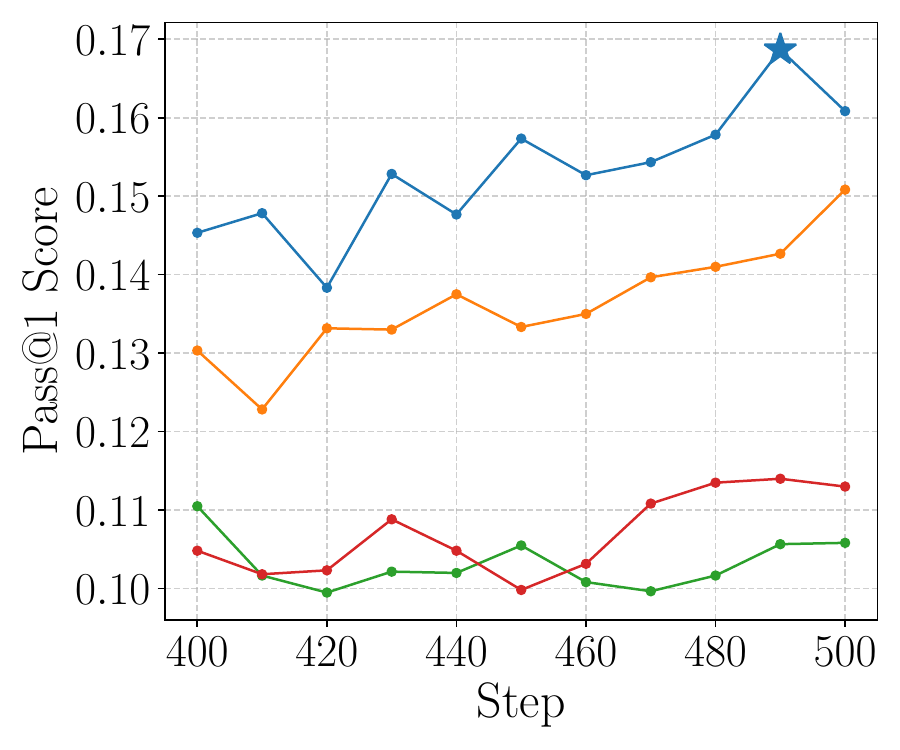}}
\subfloat[Dynamics of \(q_{\mathrm{hi}}\).]{\centering\includegraphics[trim=0cm 0.5cm 0cm 0cm,width=0.245\textwidth]{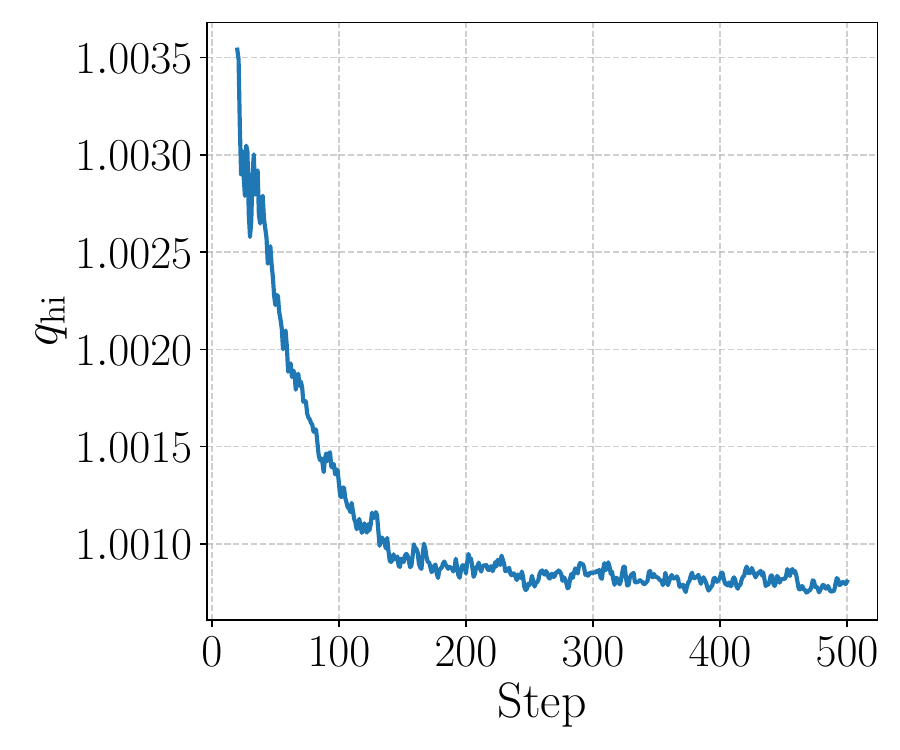}}
\subfloat[Dynamics of \(q_{\mathrm{lo}}\).]{\centering\includegraphics[trim=0cm 0.5cm 0cm 0cm,width=0.245\textwidth]{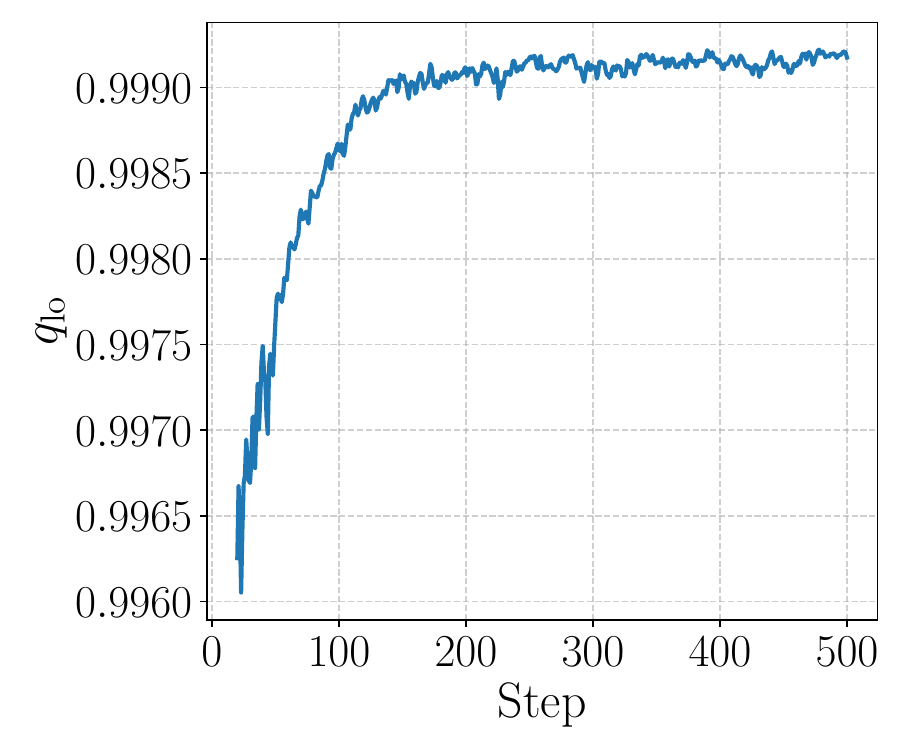}}
\caption{Learning curves and evaluation results on DAPO-Math-17k.}
\label{fig:learning_curves_dapo_math}
\end{figure*}

\paragraph{Sensitivity and computational cost.}
The main experiments use \(\alpha=0.1\) and \(K=2\), selected using preliminary ablations on MATH and GSM8K (Appendix~\ref{app:ablation_studies}). On DAPO-Math-17k, increasing \(\alpha\) from \(0.1\) to \(0.2\) and \(0.4\) reduces average Pass@1 from \(34.11\) to \(33.50\) and \(32.37\), respectively (Table~\ref{tab:hard_alpha_ablation}). For a fixed buffer and proportional allocation between the tails, a smaller \(\alpha\) widens the quantile interval and leaves more scores inside the band. The reported results favor \(\alpha=0.1\) among the tested values. Figure~\ref{fig:dapo_math_efficiency} in Appendix~\ref{app:additional_training_analysis} shows comparable training-step and actor-update times for CSPO and the baselines.

\section{Conclusion}

CSPO adapts sequence-level clipping bounds using quantiles of recent training scores. Its tail bounds account for data reuse and policy drift, and its per-update bounds control the objective and gradient error under the stated moment conditions. In our experiments, CSPO improves over GSPO on mathematical reasoning benchmarks and achieves the highest overall success rate on ALFWorld. The changing endpoints provide a direct view of how the clipping band adapts during training.

\paragraph{Limitations.}
The population tail bounds depend on discrepancies between the CDFs of retained scores and their population mixtures, as well as local CDF drift. Reusing a rollout batch does not yield a sampling rate for these discrepancies from the number of retained entries alone. The gradient bound also assumes a bound on the normalized score gradient that need not hold for a neural network. The results quantify error relative to the GSPO power surrogate; they do not establish a KL constraint or reward improvement.

\pagebreak
\subsection*{Reproducibility statement}
Algorithm~\ref{alg:cspo} specifies score retention and buffer updates; Section~\ref{sec:theory} states the assumptions; and Appendix~\ref{app:theory} provides complete proofs. Appendix~\ref{app:experimental_setup} records the training settings, and Appendix~\ref{app:ablation_studies} reports the choices of \(\alpha\) and \(K\) and their ablations.

\subsection*{AI use statement}
Generative AI tools were used to aid or polish writing.

\bibliography{iclr2027_conference}
\bibliographystyle{iclr2027_conference}

\newpage
\appendix

\section{Related Work}
Related work relevant to CSPO spans three complementary directions: reinforcement learning for language-model post-training, the transition from token-level to sequence-level policy optimization, and conformal prediction and adaptive calibration. We review these strands in turn to position CSPO as a sequence-level policy optimization method that replaces a prespecified clipping band with one calibrated from retained training scores.

\emph{Reinforcement learning for language model post-training.} 
RL-based post-training has a long history in aligning models with human or task-specific preferences. Early preference-based RL showed that complex objectives can be learned from pairwise human comparisons \citep{christiano2017deep}. In language modeling, RLHF pipelines train reward models from human preference data and optimize policies to maximize the learned reward while constraining deviation from a reference model \citep{stiennon2020learning,ouyang2022training}. Constitutional AI reduces direct human labeling by using principle-guided AI feedback for harmlessness training \citep{bai2022constitutional}. In parallel, direct preference optimization methods bypass explicit reward-model RL by deriving a supervised objective from the KL-regularized RLHF optimum \citep{rafailov2023direct}. These methods focus primarily on preference alignment, while our work targets policy optimization for verifiable reasoning rewards.

RLVR has recently become a central approach for improving mathematical and coding reasoning. When rewards are automatically verifiable, the supervision signal can be scaled without dense human annotations. DeepSeekMath introduced GRPO to improve mathematical reasoning while reducing PPO memory overhead \citep{shao2024deepseekmath}. DeepSeek-R1 further demonstrated that reasoning behaviors such as reflection and verification can emerge through large-scale RL \citep{guo2025deepseekr1}. DAPO highlights several practical ingredients for large-scale long-CoT RL, including decoupled clipping and dynamic sampling \citep{yu2025dapo}. Our method is complementary to these systems: CSPO modifies the clipping geometry of the policy objective and can be combined with different reward functions, sampling strategies, and KL regularizers.

\emph{From token-level to sequence-level policy optimization.} 
PPO stabilizes policy-gradient updates by clipping the probability ratio at the action level \citep{schulman2017proximal}. In autoregressive language models, applying this design token-by-token is convenient, but the reward and advantage are often defined for an entire generated response. GRPO keeps this token-level clipping structure while replacing the value model with group-relative sequence rewards \citep{shao2024deepseekmath}. This design is efficient, but it implicitly assigns a single response-level advantage to many token-level ratios. As a result, the objective may over-penalize locally off-policy tokens or distort the contribution of a high-value response.

Sequence-level methods address this mismatch by moving the clipping multiplier to the response level. GSPO constructs a power multiplier from the sequence likelihood ratio and performs sequence-level clipping and optimization \citep{zheng2025gspo}. FSPO further studies the length fairness problem of sequence-level clipping and shows that fixed clipping bands can reweight short and long responses differently; it introduces a \(\sqrt{L}\)-scaled clipping band for the sequence log-ratio \citep{mao2025clip}. CSPO extends this approach by calibrating the clipping band from the recent empirical distribution of normalized sequence scores, allowing the endpoints to change during training.

\emph{Conformal prediction and adaptive calibration.}
Conformal prediction provides distribution-free uncertainty quantification by calibrating prediction sets using held-out nonconformity scores \citep{vovk2005algorithmic,shafer2008tutorial,angelopoulos2023gentle}. Its central guarantee is finite-sample marginal coverage under exchangeability, independent of the internal form of the predictive model. Conformalized quantile regression combines conformal calibration with quantile regression to obtain adaptive prediction intervals under heteroscedasticity \citep{romano2019conformalized}. Adaptive conformal inference extends calibration to online environments where the data distribution may vary over time \citep{gibbs2021adaptive,gibbs2024conformal}.

We use quantile calibration to assign a tail budget to fresh sequence-level scores outside the clipping band. The GSPO power score \(W=\exp(S/L)\) plays the role of the calibration score: values in the tails indicate unusually large policy shifts. By maintaining a rolling training-score buffer, CSPO adapts the clipping interval to the current training regime while targeting a marginal tail budget with explicit CDF discrepancies and local drift.

\section{Theoretical Details}
\label{app:theory}

This appendix states the supplementary corollaries used in the main discussion and verifies all formal results under the same score, sampling unit, thresholds, and conditioning used by the main algorithm.

\subsection{Autoregressive policy and GRPO objective}
\label{app:baseline_setup}
For completeness, this subsection records the autoregressive factorization and token-level reference objective that are not part of CSPO's technical novelty. For \(\vy_i=(y_{i,1},\ldots,y_{i,L_i})\) and \(\vh_{i,\ell}=(\vx_i,\vy_{i,<\ell})\),
\begin{equation}
\pi_{\vtheta}(\vy_i\mid\vx_i)=\prod_{\ell=1}^{L_i}\pi_{\vtheta}(y_{i,\ell}\mid\vh_{i,\ell}).
\label{eq:policy_factorization}
\end{equation}
For \(G\) responses, use the group-normalized \(\widehat A_i\) and token ratio \(r_{i,\ell}(\vtheta)\) defined in Section~\ref{sec:method}. The token-level GRPO reference is
\begin{equation}
J_{\mathrm{tok}}(\vtheta)=\mathbb E\!\left[\frac1G\sum_{i=1}^{G}\frac1{L_i}\sum_{\ell=1}^{L_i}\min\!\left\{r_{i,\ell}(\vtheta)\widehat A_i,\operatorname{clip}(r_{i,\ell}(\vtheta),1-\epsilon,1+\epsilon)\widehat A_i\right\}\right].
\label{eq:token_grpo}
\end{equation}

\subsection{Detailed method positioning}
Table~\ref{tab:method_positioning} separates the object used to construct a clipping rule from the object controlled by its stated analysis.
\begin{table*}[htbp]
\centering
\caption{Comparison with representative clipping and conformal-control methods. ``Controlled object'' denotes the stated construction or guarantee, not an empirical performance claim.}
\label{tab:method_positioning}
\small
\resizebox{\textwidth}{!}{%
\begin{tabular}{c|c|c|c}
\hline
Method & Update element or score & Band/control construction & Primary controlled object \\
\hline
PPO / GRPO & Token ratio & Fixed numerical endpoints & Pessimistic token surrogate \\
GSPO & Response \(W=e^{S/L}\) & Fixed response-level endpoints & Sequence-coherent power surrogate \\
FSPO & Sequence log-ratio & Fixed \(\sqrt L\)-scaled geometry & Length reweighting and gradient direction \\
CTPO \citep{zhang2026ctpo} & Prefix/cumulative ratio & Position-dependent scaling & Prefix correction and position scaling \\
BandPO \citep{li2026bandpo} & Probability ratio & Divergence-derived probability-aware bounds & Projected trust-region geometry \\
BAPO \citep{xi2025bapo} & Ratio contributions & Advantage/entropy-driven asymmetry & Contribution balance and entropy \\
Adaptive Clip PO \citep{wang2026adaptiveclip} & Token-group ratios & Empirical-variance endpoint scaling & Gradient/update heterogeneity \\
Conformal Policy Control \citep{prinster2026cpc} & Policy relative to safe reference & Conformal control-parameter selection & Deployment-time safety loss \\
CSPO & Response \(W=e^{S/L}\) & Training-score quantiles, outward grid, and drift & Fresh-score tail mass and GSPO-surrogate fidelity \\
\hline
\end{tabular}}
\end{table*}

\subsection{Buffer size and local drift}
For a window of length \(K\), write \(r_{t,K}^{\mathcal G}\) for its realized grid discrepancy.
\begin{corollary}
\label{cor:window_tradeoff}
Suppose the conditions of Theorem~\ref{thm:grid_pooled} hold and the preceding \(K\) steps have equal numbers \(m\) of cached score entries. If, for some \(\rho\ge0\),
\[
\|F_{t-r}-F_t\|_\infty\le\rho r,\qquad r=1,\ldots,K,
\]
then, pathwise,
\begin{equation}
p_t^{\mathrm{band}}\le\min\!\left\{1,\alpha+2r_{t,K}^{\mathcal G}+\rho(K+1)\right\}.
\label{eq:window_tradeoff}
\end{equation}
For a finite set \(\mathcal K_t\) of available window lengths satisfying the same drift envelope, the corresponding oracle comparison is
\begin{equation}
K_{\mathrm{oracle}}\in\operatorname*{arg\,min}_{K\in\mathcal K_t}
\left\{2r_{t,K}^{\mathcal G}+\rho(K+1)\right\}.
\label{eq:oracle_window}
\end{equation}
The discrepancy \(r_{t,K}^{\mathcal G}\) need not decrease with \(K\); this oracle comparison is not a data-only window-selection rule.
\end{corollary}

\subsection{Secondary transfer, tracking, and path results}
This subsection bounds the raw quantile tail probability, extends the result to length groups, and analyzes quantile tracking and error accumulation across updates. The tail bound allows discrete scores and dependence across training steps. Rounding and anchoring can only reduce the probability of falling outside the band.

\begin{theorem}
\label{thm:drift}
For a step \(t\) with \(m_t,n_t>0\), define
\begin{equation}
\varepsilon_t=\sum_{\tau\in\mathcal I_t}\frac{m_\tau}{n_t}\xi_\tau,\qquad \xi_\tau=\|\widehat F_\tau-F_\tau\|_\infty.
\label{eq:buffer_error}
\end{equation}
Then the first two bounds below hold on every realized training path; the last also uses Assumption~\ref{assump:local_stability}:
\begin{equation}
\begin{aligned}
p_t^{\mathrm{band}}
&:=
\mathbb{P}\!\left(
W_t^\star<a_t
\ \mathrm{or}\
W_t^\star>b_t
\,\middle|\,
\mathcal H_t
\right)\\
&\le
p_t^{q}
:=
\mathbb{P}\!\left(
W_t^\star<q_{\mathrm{lo},t}
\ \mathrm{or}\
W_t^\star>q_{\mathrm{hi},t}
\,\middle|\,
\mathcal H_t
\right)\\
&\le
\alpha+2\varepsilon_t+2D_t
\le
\alpha+2\varepsilon_t+2\Delta_K.
\end{aligned}
\label{eq:drift_control}
\end{equation}
\end{theorem}

Theorem~\ref{thm:drift} is intentionally one-sided: deterministic quantiles can be conservative when an endpoint has large probability mass. Let \(\omega_t=\sup_w\{F_t(w)-F_t(w^-)\}\) be the largest current atom. Corollary~\ref{cor:drift_atoms} shows that the same path also satisfies
\[
|p_t^q-\alpha|
\le
2\bigl(\varepsilon_t+D_t\bigr)+2\omega_t.
\]
The upper bound allows ties under local CDF drift, while closeness to the target from both sides also depends on endpoint mass.
The marginal theorem averages over the response-length distribution. To determine when it also describes a length subgroup, let \(H_t^\star=h(L_t^\star)\) be a pre-specified length bin with positive conditional probability and define
\[
F_{t,h}(w)
=
\mathbb{P}\!\left(
W_t^\star\le w
\,\middle|\,
\mathcal H_t,H_t^\star=h
\right),
\qquad
\Gamma_{t,h}
=
\sup_w|F_{t,h}(w)-F_t(w)|.
\]
The quantity \(\Gamma_{t,h}\) measures the CDF discrepancy between length group \(h\) and the full mixture; normalization by \(L\) alone does not force it to be small.

\begin{proposition}
\label{prop:length_group_fairness}
Let
\[
p_{t,h}^{\mathrm{band}}
=
\mathbb{P}\!\left(
W_t^\star<a_t
\ \mathrm{or}\
W_t^\star>b_t
\,\middle|\,
\mathcal H_t,H_t^\star=h
\right).
\]
Then, almost surely,
\begin{equation}
|p_{t,h}^{\mathrm{band}}-p_t^{\mathrm{band}}|
\le
2\Gamma_{t,h}.
\label{eq:length_group_transfer}
\end{equation}
\end{proposition}

Combining Proposition~\ref{prop:length_group_fairness} with Theorem~\ref{thm:drift} gives the group upper bound
\[
p_{t,h}^{\mathrm{band}}
\le
\alpha+2\bigl(\varepsilon_t+D_t\bigr)+2\Gamma_{t,h}.
\]
Corollary~\ref{cor:length_group_control} also gives the atom-corrected two-sided statement when anchoring is inactive.

These results separate marginal calibration from length-stratified calibration. A small CDF discrepancy \(\Gamma_{t,h}\) supports transfer to group \(h\).
Coverage alone does not quantify how closely the rolling band follows the current quantiles. The next assumption adds the local identifiability needed to convert a CDF error into endpoint error; it is not required by Theorems~\ref{thm:drift} or~\ref{thm:grid_pooled}.

\begin{assumption}
\label{assump:quantile_identifiability}
For \(p\in\{\alpha_{\mathrm{lo}},1-\alpha_{\mathrm{hi}}\}\) and every analyzed round, let \(q_{t,p}^\star=F_t^{-1}(p):=\inf\{w:F_t(w)\ge p\}\). There exist \(\underline f>0\) and \(r_q>0\) such that \(F_t\) is continuous on \([q_{t,p}^\star-r_q,q_{t,p}^\star+r_q]\) and
\[
|F_t(q_{t,p}^\star+u)-p|
\ge
\underline f|u|,
\qquad |u|\le r_q.
\]
Every such neighborhood lies strictly inside the finite grid range. For a predeclared resolution \(h_{\mathcal G}\ge0\), every \(z\) in such a neighborhood has grid points \(g^-(z),g^+(z)\in\mathcal G_M\) satisfying
\[
g^-(z)\le z\le g^+(z),
\qquad
z-g^-(z)\le h_{\mathcal G},
\qquad
g^+(z)-z\le h_{\mathcal G}.
\]
Along each analyzed path, the CDF discrepancies below obey \(r_t^{\mathcal G}+D_t<\underline f r_q\).
\end{assumption}

Assumption~\ref{assump:quantile_identifiability} excludes flat CDF regions around the target quantiles. Under this condition, the following result bounds the cumulative error in tracking the quantiles of the score mixture.

\begin{proposition}
\label{prop:dynamic_tracking}
Suppose Assumptions~\ref{assump:local_stability} and~\ref{assump:quantile_identifiability} hold for rounds \(t=K+1,\ldots,T\), with equal block size \(m\), a fixed horizon \(K\), and the grid endpoints used by Algorithm~\ref{alg:cspo}. Let
\[
v_t=\|F_t-F_{t-1}\|_\infty,
\qquad
V_T=\sum_{t=2}^{T}v_t.
\]
Set \(r_{T,K}^{\mathcal G}=\max_{K<t\le T}r_t^{\mathcal G}\). For either mixture quantile level \(p\), let \(\widetilde q_{t,p}\) be the corresponding outward-rounded empirical endpoint. On every path satisfying the stated conditions,
\begin{equation}
\sum_{t=K+1}^{T}
|\widetilde q_{t,p}-q_{t,p}^\star|
\le
(T-K)h_{\mathcal G}
+
\frac{1}{\underline f}
\left[
(T-K)r_{T,K}^{\mathcal G}
+
\frac{K+1}{2}V_T
\right].
\label{eq:dynamic_quantile_tracking}
\end{equation}
If \(\mathbb E_{F_t}[W]<\infty\) for every analyzed round, then for the pinball risk \(\mathcal L_{t,p}(q)=\mathbb E_{F_t}[\rho_p(W-q)]\), \(\rho_p(u)=u(p-\mathbf 1\{u<0\})\), the cumulative dynamic regret \(\sum_t[\mathcal L_{t,p}(\widetilde q_{t,p})-\mathcal L_{t,p}(q_{t,p}^\star)]\) is bounded by the same right-hand side.

Let \((a_t^\star,b_t^\star)\) be the oracle quantile endpoints anchored at \(1\). Let \(\mathcal J_{\mathrm G,t}=\sum_{k\in\mathcal C_t}\nu_{t,k}\mathcal J_{\mathrm G}(\vtheta_{t,k})\) and \(\mathcal J_{\mathrm C,t}^{a,b}=\sum_{k\in\mathcal C_t}\nu_{t,k}\mathcal J_{\mathrm C}^{a,b}(\vtheta_{t,k})\) be the corresponding collection-weighted population surrogates, all with denominator \(\pi_{\vtheta_{t,0}}\). Define \(\mathfrak D_t(a,b)=\mathcal J_{\mathrm G,t}-\mathcal J_{\mathrm C,t}^{a,b}\), keeping these weights fixed when comparing bands. If, in addition, the finite-endpoint and moment conditions of Proposition~\ref{prop:objective_distortion} hold at every collected update and \(|\widehat A|\le A_{\max}\), then
\begin{equation}
\sum_{t=K+1}^{T}
|\mathfrak D_t(a_t,b_t)-\mathfrak D_t(a_t^\star,b_t^\star)|
\le
2A_{\max}
\left\{
(T-K)h_{\mathcal G}
+
\frac{(T-K)r_{T,K}^{\mathcal G}+(K+1)V_T/2}{\underline f}
\right\}.
\label{eq:dynamic_distortion_tracking}
\end{equation}
\end{proposition}

The dynamic regret in Proposition~\ref{prop:dynamic_tracking} establishes that CSPO follows its chosen quantile target under drift; it does not claim that this target maximizes environment reward. In particular, an unbounded band trivially has zero clipping distortion but abandons the intended regularization budget. The empirical evaluation must therefore report both calibration fidelity and task performance.

The preceding results operate across outer rounds. To state what tail control would imply within one rollout step, we finally compare two exact-gradient paths for one fixed old policy, advantage law, and surrogate. In this path comparison, the fixed outer-step index is suppressed. This result also identifies the missing condition: a bound at one parameter value does not certify every inner iterate.

\begin{proposition}
\label{prop:path_deviation}
Let
\[
\vtheta_{k+1}^{\mathrm G}
=\vtheta_k^{\mathrm G}+\gamma_k g_{\mathrm G}(\vtheta_k^{\mathrm G}),
\qquad
\vtheta_{k+1}^{\mathrm C}
=\vtheta_k^{\mathrm C}+\gamma_k g_{\mathrm C}(\vtheta_k^{\mathrm C}),
\qquad
\vtheta_0^{\mathrm G}=\vtheta_0^{\mathrm C}.
\]
Suppose \(\gamma_k\ge0\), \(B_k\ge0\), \(g_{\mathrm G}\) is \(L_g\)-Lipschitz on a set containing both paths, and \(\|g_{\mathrm C}(\vtheta_k^{\mathrm C})-g_{\mathrm G}(\vtheta_k^{\mathrm C})\|\le B_k\). Then, for every \(E\ge1\),
\begin{equation}
\|\vtheta_E^{\mathrm C}-\vtheta_E^{\mathrm G}\|
\le
\sum_{j=0}^{E-1}
\gamma_jB_j
\prod_{\ell=j+1}^{E-1}(1+\gamma_\ell L_g).
\label{eq:path_deviation}
\end{equation}
If Proposition~\ref{prop:gradient_fidelity} is valid at every CSPO iterate and \(p_k\) is the fresh-score risk at \(\vtheta_k^{\mathrm C}\) relative to the same fixed old policy, band, and advantage law, one may take \(B_k=A_{\max}M_\psi\sqrt{(1+2/L_{\min})p_k}\). This substitution requires the absolute-continuity, length, advantage, and score-gradient conditions at every iterate, together with a simultaneous bound on the risks \((p_k)_k\) along the path. The statement concerns deterministic exact-gradient paths, not an SGD trajectory.
\end{proposition}

\subsection{Supplementary calibration and optimization corollaries}

The main theorem gives the tie-robust upper guarantee used by the operational algorithm. To characterize when the raw quantile interval is also close to the nominal tail mass from below, the following corollary makes endpoint atoms explicit.

\begin{corollary}
\label{cor:drift_atoms}
Under the conditions of Theorem~\ref{thm:drift}, define
\begin{equation}
\omega_t
=
\sup_{w\in\mathbb{R}}
\{F_t(w)-F_t(w^-)\},
\qquad
F_t(w^-)=\lim_{u\uparrow w}F_t(u).
\label{eq:max_atom}
\end{equation}
On the same realized path,
\begin{equation}
\alpha
-
2\bigl(\varepsilon_t+D_t\bigr)
-
2\omega_t
\le
p_t^{q}
\le
\alpha
+
2\bigl(\varepsilon_t+D_t\bigr).
\label{eq:drift_control_atoms}
\end{equation}
Consequently,
\begin{equation}
|p_t^{q}-\alpha|
\le
2\bigl(\varepsilon_t+D_t\bigr)+2\omega_t.
\label{eq:drift_absolute_atoms}
\end{equation}
\end{corollary}

The length-transfer proposition isolates the discrepancy between a group CDF and the marginal CDF. Combining it with the calibration theorem yields the following operational group guarantee.

\begin{corollary}
\label{cor:length_group_control}
Under the conditions of Theorem~\ref{thm:drift}, every length group \(h\) with positive conditional probability satisfies the following pathwise bound:
\begin{equation}
p_{t,h}^{\mathrm{band}}
\le
\alpha
+
2\bigl(\varepsilon_t+D_t\bigr)
+
2\Gamma_{t,h}.
\label{eq:length_group_upper}
\end{equation}
If rounding and anchoring leave the raw interval unchanged, so that \((a_t,b_t)=(q_{\mathrm{lo},t},q_{\mathrm{hi},t})\), then
\begin{equation}
|p_{t,h}^{\mathrm{band}}-\alpha|
\le
2\bigl(\varepsilon_t+D_t\bigr)
+
2\omega_t
+
2\Gamma_{t,h}.
\label{eq:length_group_absolute}
\end{equation}
\end{corollary}

The gradient-fidelity proposition gives a positive inner product when the unclipped gradient dominates its calibrated tail error. The following auxiliary statement applies when its gradient estimator is conditionally unbiased and independent across prompt groups; these properties are not inferred for a reused-batch update at \(k>0\).

\begin{corollary}
\label{cor:one_step_power_improvement}
Let \(\mathcal J\) be a sigma-field fixing the parameter pair and band before the gradient-estimator randomness. Assume \(L_{\mathrm s}>0\) and that \(\mathcal{J}_{\mathrm{G}}\) is \(L_{\mathrm s}\)-smooth on a convex set containing \(\vtheta\) and \(\vtheta+\gamma\widehat g_{\mathrm{C}}\) almost surely for every step size considered. Let
\[
C_{\mathrm{grad}}
=
A_{\max}^{2}M_{\psi}^{2}
\left(1+\frac{2}{L_{\min}}\right).
\]
For prompt group \(j\), let \(X_{j,i}\) denote the per-response selected direction \(X_{\mathrm{C}}\) from Proposition~\ref{prop:gradient_fidelity} and define
\[
Z_j=\frac{1}{G}\sum_{i=1}^{G}X_{j,i}.
\]
Let \(Z_1,\ldots,Z_N\) be conditionally independent across prompt groups given \(\mathcal J\), with \(\mathbb{E}[Z_j\mid\mathcal J]=g_{\mathrm{C}}\), and set \(\widehat g_{\mathrm{C}}=N^{-1}\sum_{j=1}^{N}Z_j\). Then
\[
\mathbb{E}[
\|\widehat g_{\mathrm{C}}-g_{\mathrm{C}}\|^2
\mid\mathcal J]
\le
\frac{C_{\mathrm{grad}}}{N}.
\]
For \(\vtheta^{+}=\vtheta+\gamma\widehat g_{\mathrm{C}}\),
\begin{align}
\mathbb{E}\!\left[
\mathcal{J}_{\mathrm{G}}(\vtheta^{+})
-
\mathcal{J}_{\mathrm{G}}(\vtheta)
\,\middle|\,
\mathcal J
\right]
\ge{}&
\gamma\|g_{\mathrm{G}}\|
(\|g_{\mathrm{G}}\|-B_{\vtheta})
-
\frac{L_{\mathrm s}\gamma^2}{2}
\left[
(\|g_{\mathrm{G}}\|+B_{\vtheta})^2
+
\frac{C_{\mathrm{grad}}}{N}
\right].
\label{eq:one_step_power_improvement}
\end{align}
If \(B_{\vtheta}<\|g_{\mathrm{G}}\|\), the lower bound is positive for every
\begin{equation}
0<\gamma<
\frac{
2\|g_{\mathrm{G}}\|
(\|g_{\mathrm{G}}\|-B_{\vtheta})
}{
L_{\mathrm s}[
(\|g_{\mathrm{G}}\|+B_{\vtheta})^2
+C_{\mathrm{grad}}/N
]
}.
\label{eq:one_step_stepsize}
\end{equation}
\end{corollary}

The rolling theorem is intentionally nonexchangeable. To identify the stronger conformal benchmark precisely, consider instead calibration and future scores evaluated by one policy pair that was fixed before either was sampled. The rank construction below is needed only for this stationary reference case and is not used to justify pooling across policy pairs.

\begin{theorem}
\label{thm:exchangeable_rank}
Condition on a sigma-field \(\mathcal G\) with respect to which a policy pair and its score map are fixed. Suppose positive scores \(W_1,\ldots,W_n,W^\star\) are exchangeable conditional on \(\mathcal G\). Let \(W_{(1)}\le\cdots\le W_{(n)}\), set \(W_{(0)}=0\) and \(W_{(n+1)}=+\infty\), and define
\[
k_{\mathrm{lo}}=\lfloor(n+1)\alpha_{\mathrm{lo}}\rfloor,
\qquad
k_{\mathrm{hi}}=\lceil(n+1)(1-\alpha_{\mathrm{hi}})\rceil.
\]
With
\[
q_{\mathrm{lo}}=W_{(k_{\mathrm{lo}})},
\qquad
q_{\mathrm{hi}}=W_{(k_{\mathrm{hi}})},
\qquad
a=\min\{q_{\mathrm{lo}},1\},
\qquad
b=\max\{q_{\mathrm{hi}},1\},
\]
one has, for arbitrary discrete or continuous score laws,
\begin{equation}
\mathbb P(W^\star<a\ \mathrm{or}\ W^\star>b\mid\mathcal G)
\le
\frac{k_{\mathrm{lo}}}{n+1}
+
\frac{n+1-k_{\mathrm{hi}}}{n+1}
\le\alpha.
\label{eq:exchangeable_rank}
\end{equation}
Strict tail events make ties conservative. Equality at the nominal level requires randomized tie and boundary-rank handling and is not claimed here.
\end{theorem}

The current completed score block permits an empirical comparison with the retained blocks.
\begin{corollary}
\label{cor:observable_drift}
For \(m_t,n_t>0\), define the empirical tail probability for the step mixture and the drift diagnostic
\begin{equation}
\widehat p_t^{\mathrm{band}}=\widehat F_t(a_t^-)+1-\widehat F_t(b_t),\qquad
\widehat D_t=\sum_{\tau\in\mathcal I_t}\frac{m_\tau}{n_t}\|\widehat F_\tau-\widehat F_t\|_\infty.
\label{eq:observable_drift}
\end{equation}
These use the actual cached entries, including repeated entries. For the quantile band and its outward-rounded, anchored version, pathwise,
\begin{equation}
\widehat p_t^{\mathrm{band}}\le\min\{1,\alpha+2\widehat D_t\},\qquad
p_t^{\mathrm{band}}\le\min\{1,\alpha+2\widehat D_t+2\xi_t\}.
\label{eq:observable_certificate}
\end{equation}
Also \(D_t\le\widehat D_t+\varepsilon_t+\xi_t\). If \(\widehat\omega_t=\sup_w\{\widehat F_t(w)-\widehat F_t(w^-)\}\), then
\begin{equation}
\omega_t\le\min\{1,\widehat\omega_t+2\xi_t\}.
\label{eq:observable_atom}
\end{equation}
The empirical bound is directly computable. The population and atom bounds additionally contain the realized discrepancy \(\xi_t\); they are not data-only confidence bounds.
\end{corollary}

\begin{corollary}
\label{cor:anytime_grid}
Along any training path, simultaneously at every step with \(m_t,n_t>0\),
\begin{equation}
p_t^{\mathrm{band}}\le\min\{1,\alpha+2r_t^{\mathcal G}+2D_{t,\mathcal G}\}.
\label{eq:anytime_grid}
\end{equation}
The same inequality therefore holds at a finite step selected from the observed training history, retaining that step's realized discrepancy and drift.
\end{corollary}

\subsection{Proof of Theorem~\ref{thm:drift}}

\begin{proof}
Fix the realized training path. By \eqref{eq:step_score_mixture}, the empirical and population step CDFs use the same collection weights, so
\[
\widehat F_\tau-F_\tau
=\sum_{k\in\mathcal C_\tau}\nu_{\tau,k}(\widehat F_{\tau,k}-F_{\tau,k}),\qquad
\xi_\tau\le\sum_{k\in\mathcal C_\tau}\nu_{\tau,k}\xi_{\tau,k}.
\]
This identity holds for repeated or dependent entries as well as for distinct entries.

Let
\[
\lambda_{\tau,t}=\frac{m_\tau}{n_t},
\qquad
\overline F_t(w)
=
\sum_{\tau\in\mathcal{I}_t}
\lambda_{\tau,t}F_\tau(w).
\]
Since
\[
\widehat F_{\mathcal{B}_t}(w)
=
\sum_{\tau\in\mathcal{I}_t}
\lambda_{\tau,t}\widehat F_\tau(w),
\]
pathwise,
\[
\sup_w
|\widehat F_{\mathcal{B}_t}(w)-\overline F_t(w)|
\le
\sum_{\tau\in\mathcal{I}_t}
\lambda_{\tau,t}\xi_\tau
=
\varepsilon_t.
\]
When Assumption~\ref{assump:local_stability} holds, it additionally gives
\[
\sup_w|\overline F_t(w)-F_t(w)|
=
D_t
\le
\sum_{\tau\in\mathcal{I}_t}
\lambda_{\tau,t}
\sup_w|F_\tau(w)-F_t(w)|
\le
\Delta_K.
\]
Thus, on the same path,
\begin{equation}
\sup_w
|\widehat F_{\mathcal{B}_t}(w)-F_t(w)|
\le
\eta_t,
\qquad
\eta_t:=\varepsilon_t+D_t.
\label{eq:proof_uniform_cdf}
\end{equation}

By the definition of the empirical inverse CDF,
\[
\widehat F_{\mathcal{B}_t}(q_{\mathrm{lo},t}^{-})
\le
\alpha_{\mathrm{lo}},
\qquad
\widehat F_{\mathcal{B}_t}(q_{\mathrm{hi},t})
\ge
1-\alpha_{\mathrm{hi}}.
\]
Taking left limits in \eqref{eq:proof_uniform_cdf} is valid for CDFs, so
\[
F_t(q_{\mathrm{lo},t}^{-})
\le
\alpha_{\mathrm{lo}}+\eta_t,
\qquad
1-F_t(q_{\mathrm{hi},t})
\le
\alpha_{\mathrm{hi}}+\eta_t.
\]
The empirical quantiles are ordered because \(\alpha_{\mathrm{lo}}<1-\alpha_{\mathrm{hi}}\). Therefore the strict tail events are disjoint and
\[
p_t^q
=
F_t(q_{\mathrm{lo},t}^{-})
+
1-F_t(q_{\mathrm{hi},t})
\le
\alpha+2\eta_t.
\]
Finally, \(a_t\le q_{\mathrm{lo},t}\) and \(b_t\ge q_{\mathrm{hi},t}\), so
\[
\{W_t^\star<a_t\ \mathrm{or}\ W_t^\star>b_t\}
\subseteq
\{W_t^\star<q_{\mathrm{lo},t}\ \mathrm{or}\ W_t^\star>q_{\mathrm{hi},t}\}.
\]
This proves both inequalities in \eqref{eq:drift_control}.
\end{proof}

\subsection{Proof of Corollary~\ref{cor:drift_atoms}}

\begin{proof}
Retain \(\eta_t=\varepsilon_t+D_t\) from the preceding pathwise proof. The inverse empirical quantiles also satisfy
\[
\widehat F_{\mathcal{B}_t}(q_{\mathrm{lo},t})
\ge
\alpha_{\mathrm{lo}},
\qquad
\widehat F_{\mathcal{B}_t}(q_{\mathrm{hi},t}^{-})
\le
1-\alpha_{\mathrm{hi}}.
\]
Using \eqref{eq:proof_uniform_cdf},
\[
F_t(q_{\mathrm{lo},t})
\ge
\alpha_{\mathrm{lo}}-\eta_t,
\qquad
F_t(q_{\mathrm{hi},t}^{-})
\le
1-\alpha_{\mathrm{hi}}+\eta_t.
\]
The mass at either endpoint is at most \(\omega_t\). Hence
\[
F_t(q_{\mathrm{lo},t}^{-})
\ge
\alpha_{\mathrm{lo}}-\eta_t-\omega_t
\]
and
\[
1-F_t(q_{\mathrm{hi},t})
\ge
\alpha_{\mathrm{hi}}-\eta_t-\omega_t.
\]
Adding these bounds gives \(p_t^q\ge\alpha-2\eta_t-2\omega_t\). The upper bound follows from Theorem~\ref{thm:drift}, and the absolute-deviation statement follows by combining the two sides.
\end{proof}

\subsection{Proof of Proposition~\ref{prop:length_group_fairness}}

\begin{proof}
Conditional on \(\mathcal H_t\),
\[
p_t^{\mathrm{band}}
=
F_t(a_t^-)+1-F_t(b_t)
\]
and
\[
p_{t,h}^{\mathrm{band}}
=
F_{t,h}(a_t^-)+1-F_{t,h}(b_t).
\]
Consequently,
\[
|p_{t,h}^{\mathrm{band}}-p_t^{\mathrm{band}}|
\le
|F_{t,h}(a_t^-)-F_t(a_t^-)|
+
|F_{t,h}(b_t)-F_t(b_t)|.
\]
The Kolmogorov bound defining \(\Gamma_{t,h}\) also controls left limits, by taking an increasing sequence converging to \(a_t\). Each term is at most \(\Gamma_{t,h}\), which proves \eqref{eq:length_group_transfer}.
\end{proof}

\subsection{Proof of Corollary~\ref{cor:length_group_control}}

\begin{proof}
By the pathwise bound in Theorem~\ref{thm:drift},
\[
p_{t,h}^{\mathrm{band}}
\le
p_t^{\mathrm{band}}+2\Gamma_{t,h}
\le
\alpha+2\bigl(\varepsilon_t+D_t\bigr)+2\Gamma_{t,h},
\]
which proves \eqref{eq:length_group_upper}. If the operational and raw intervals coincide, then \(p_t^{\mathrm{band}}=p_t^q\), and
\[
|p_{t,h}^{\mathrm{band}}-\alpha|
\le
|p_{t,h}^{\mathrm{band}}-p_t^{\mathrm{band}}|
+
|p_t^q-\alpha|.
\]
Proposition~\ref{prop:length_group_fairness} and Corollary~\ref{cor:drift_atoms} now give \eqref{eq:length_group_absolute}.
\end{proof}

\subsection{Proof of Lemma~\ref{lem:power_ratio_second_moment}}

\begin{proof}
Write \(R=R_{\vtheta}\) and set \(\gamma_{\max}=2/L_{\min}\le1\). If \(0\le R\le1\), then \(R^{2/L}\le1\). If \(R>1\), then \(2/L\le\gamma_{\max}\), and concavity of \(r\mapsto r^{\gamma_{\max}}\) gives
\[
R^{2/L}
\le
R^{\gamma_{\max}}
\le
1+\gamma_{\max}(R-1).
\]
Thus, pointwise,
\[
R^{2/L}
\le
1+\gamma_{\max}(R-1)_+.
\]
Because \(\vy_U\) has the old-policy marginal, absolute continuity and normalization of the exact sequence ratio imply
\[
\mathbb{E}_{\vtheta_{\mathrm{old}}}[R]
=
\mathbb{E}_{\vx\sim p}
\left[
\sum_{\vy}
\pi_{\vtheta_{\mathrm{old}}}(\vy\mid\vx)
\frac{\pi_{\vtheta}(\vy\mid\vx)}
     {\pi_{\vtheta_{\mathrm{old}}}(\vy\mid\vx)}
\right]
=1.
\]
Therefore \(\mathbb{E}[(R-1)_+]\le\mathbb{E}[R]=1\), and
\[
\mathbb{E}[W_{\vtheta}^{2}]
=
\mathbb{E}[R^{2/L}]
\le
1+\gamma_{\max}
=
1+\frac{2}{L_{\min}}.
\]
\end{proof}

\subsection{Proof of Proposition~\ref{prop:objective_distortion}}

\begin{proof}
If \(\widehat A\ge0\), then
\[
\ell_{a,b}(W_{\vtheta},\widehat A)
=
\widehat A\min\{W_{\vtheta},b\},
\]
so the pointwise gap is \(\widehat A_+(W_{\vtheta}-b)_+\). If \(\widehat A<0\), multiplication reverses the ordering and
\[
\ell_{a,b}(W_{\vtheta},\widehat A)
=
\widehat A\max\{W_{\vtheta},a\},
\]
so the gap is \(\widehat A_-(a-W_{\vtheta})_+\). This proves the nonnegative identity in \eqref{eq:exact_objective_distortion}.

For the upper tail, Cauchy--Schwarz and Lemma~\ref{lem:power_ratio_second_moment} give
\[
\begin{aligned}
\mathbb{E}
[\widehat A_+(W_{\vtheta}-b)_+]
&\le
A_{\max}\mathbb{E}
[W_{\vtheta}\mathbf{1}_{\{W_{\vtheta}>b\}}]\le
A_{\max}
\sqrt{
\mathbb{E}[W_{\vtheta}^2]\,
p_{\vtheta}^{+}
}\le
A_{\max}
\sqrt{
\left(1+\frac{2}{L_{\min}}\right)
p_{\vtheta}^{+}
}.
\end{aligned}
\]
For the lower tail,
\[
\mathbb{E}
[\widehat A_-(a-W_{\vtheta})_+]
\le
A_{\max}a\,p_{\vtheta}^{-}.
\]
Combining the two bounds proves \eqref{eq:objective_distortion_tailwise}; using \(p_{\vtheta}^{+},p_{\vtheta}^{-}\le p_{\vtheta}\) proves \eqref{eq:objective_distortion_total}.
\end{proof}

\subsection{Proof of Proposition~\ref{prop:gradient_fidelity}}

\begin{proof}
Because the old policy and \(L\) are fixed within the update,
\[
\nabla_{\vtheta}W_{\vtheta}
=
W_{\vtheta}\nabla_{\vtheta}\log W_{\vtheta}
=
W_{\vtheta}\frac{1}{L}
\nabla_{\vtheta}\log\pi_{\vtheta}(\vy\mid\vx)
=
W_{\vtheta}\psi_{\vtheta}.
\]
Differentiating the unclipped objective gives the identity for \(g_{\mathrm{G}}\) in \eqref{eq:power_gradients}. For \(\widehat A>0\), the clipped term is constant only when \(W_{\vtheta}>b\); for \(\widehat A<0\), it is constant only when \(W_{\vtheta}<a\). At equality, the specified branch rule selects the unclipped derivative. This gives the identity for the selected direction \(g_{\mathrm{C}}\) and
\[
g_{\mathrm{G}}-g_{\mathrm{C}}
=
\mathbb{E}
[W_{\vtheta}\widehat A\psi_{\vtheta}
\mathbf{1}_{\mathcal{C}_{\vtheta}}].
\]
Using the advantage and score-vector bounds, Cauchy--Schwarz, the event inclusion \(\mathbb{P}(\mathcal{C}_{\vtheta})\le p_{\vtheta}\), and Lemma~\ref{lem:power_ratio_second_moment},
\[
\begin{aligned}
\|g_{\mathrm{G}}-g_{\mathrm{C}}\|
&\le
A_{\max}M_{\psi}
\mathbb{E}
[W_{\vtheta}\mathbf{1}_{\mathcal{C}_{\vtheta}}]\le
A_{\max}M_{\psi}
\sqrt{
\mathbb{E}[W_{\vtheta}^{2}]
\mathbb{P}(\mathcal{C}_{\vtheta})
}\le
A_{\max}M_{\psi}
\sqrt{
\left(1+\frac{2}{L_{\min}}\right)p_{\vtheta}
}.
\end{aligned}
\]
Finally,
\[
\|X_{\mathrm{C}}\|^2
\le
A_{\max}^2M_{\psi}^2W_{\vtheta}^2.
\]
Taking expectations and invoking the lemma proves \eqref{eq:clipped_gradient_second_moment}.
\end{proof}

\subsection{Proof of Corollary~\ref{cor:one_step_power_improvement}}

\begin{proof}
For the group estimator in the corollary, Jensen's inequality gives
\[
\|Z_j\|^2
=
\left\|
\frac{1}{G}\sum_{i=1}^{G}X_{j,i}
\right\|^2
\le
\frac{1}{G}\sum_{i=1}^{G}\|X_{j,i}\|^2.
\]
Equation~\eqref{eq:clipped_gradient_second_moment} therefore implies \(\mathbb{E}[\|Z_j\|^2\mid\mathcal J]\le C_{\mathrm{grad}}\); no within-group independence is required. Conditional independence across prompt groups and \(\mathbb{E}[Z_j\mid\mathcal J]=g_{\mathrm{C}}\) yield
\[
\mathbb{E}
[\|\widehat g_{\mathrm{C}}-g_{\mathrm{C}}\|^2\mid\mathcal J]
=
\frac{1}{N^2}
\sum_{j=1}^{N}
\mathbb{E}[\|Z_j-g_{\mathrm{C}}\|^2\mid\mathcal J]
\le
\frac{C_{\mathrm{grad}}}{N}.
\]
By \(L_{\mathrm s}\)-smoothness,
\[
\mathcal{J}_{\mathrm{G}}(\vtheta^+)
\ge
\mathcal{J}_{\mathrm{G}}(\vtheta)
+
\gamma\langle g_{\mathrm{G}},\widehat g_{\mathrm{C}}\rangle
-
\frac{L_{\mathrm s}\gamma^2}{2}
\|\widehat g_{\mathrm{C}}\|^2.
\]
Taking conditional expectations yields
\[
\mathbb{E}
[\mathcal{J}_{\mathrm{G}}(\vtheta^+)-\mathcal{J}_{\mathrm{G}}(\vtheta)\mid\mathcal J]
\ge
\gamma\langle g_{\mathrm{G}},g_{\mathrm{C}}\rangle
-
\frac{L_{\mathrm s}\gamma^2}{2}
\mathbb{E}[\|\widehat g_{\mathrm{C}}\|^2\mid\mathcal J].
\]
Proposition~\ref{prop:gradient_fidelity} gives
\[
\langle g_{\mathrm{G}},g_{\mathrm{C}}\rangle
\ge
\|g_{\mathrm{G}}\|
(\|g_{\mathrm{G}}\|-B_{\vtheta})
\]
and \(\|g_{\mathrm{C}}\|\le\|g_{\mathrm{G}}\|+B_{\vtheta}\). Unbiasedness and the variance bound imply
\[
\mathbb{E}[\|\widehat g_{\mathrm{C}}\|^2\mid\mathcal J]
=
\|g_{\mathrm{C}}\|^2
+
\mathbb{E}[\|\widehat g_{\mathrm{C}}-g_{\mathrm{C}}\|^2\mid\mathcal J]
\le
(\|g_{\mathrm{G}}\|+B_{\vtheta})^2
+
\frac{C_{\mathrm{grad}}}{N}.
\]
Substitution proves \eqref{eq:one_step_power_improvement}. Solving for the positive range of \(\gamma\) gives \eqref{eq:one_step_stepsize}.
\end{proof}

\subsection{Proof of Theorem~\ref{thm:exchangeable_rank}}
\begin{proof}
Condition throughout on \(\mathcal G\). Introduce independent continuous auxiliary variables only for the proof and order the \(n+1\) pairs \((W_i,U_i)\) lexicographically. Exchangeability makes the rank of \((W^\star,U^\star)\) uniform on \(\{1,\ldots,n+1\}\). The strict event \(W^\star<W_{(k_{\mathrm{lo}})}\) implies that this randomized rank is at most \(k_{\mathrm{lo}}\), while \(W^\star>W_{(k_{\mathrm{hi}})}\) implies that it exceeds \(k_{\mathrm{hi}}\). Hence their probabilities are at most \(k_{\mathrm{lo}}/(n+1)\) and \((n+1-k_{\mathrm{hi}})/(n+1)\), respectively. The two strict tails are disjoint. The definitions of the floor and ceiling ranks bound their sum by \(\alpha_{\mathrm{lo}}+\alpha_{\mathrm{hi}}=\alpha\). Finally, anchoring at \(1\) only enlarges the interval and can only remove tail events.
\end{proof}

\subsection{Proof of Corollary~\ref{cor:observable_drift}}
\begin{proof}
The weighted triangle inequality gives \(\|\widehat F_{\mathcal B_t}-\widehat F_t\|_\infty\le\widehat D_t\). The strict empirical tails of the buffer at its raw quantile endpoints sum to at most \(\alpha\), and outward rounding and anchoring can only reduce this sum. Comparing the two endpoint CDF values with \(\widehat F_t\) therefore gives \(\widehat p_t^{\mathrm{band}}\le\alpha+2\widehat D_t\). Comparing \(\widehat F_t\) with \(F_t\) adds at most \(2\xi_t\), proving \eqref{eq:observable_certificate}.

For each retained step, \(\|F_\tau-F_t\|_\infty\le\|\widehat F_\tau-\widehat F_t\|_\infty+\xi_\tau+\xi_t\). Averaging yields \(D_t\le\widehat D_t+\varepsilon_t+\xi_t\). Finally,
\[
F_t(w)-F_t(w^-)\le\widehat F_t(w)-\widehat F_t(w^-)+2\xi_t.
\]
Taking the supremum and capping by one proves \eqref{eq:observable_atom}.
\end{proof}

\subsection{Proof of Corollary~\ref{cor:anytime_grid}}
\begin{proof}
Theorem~\ref{thm:grid_pooled} is a deterministic inequality for each realized buffer and target mixture. Applying it at every step yields the pathwise collection of inequalities; selecting a finite step selects one of them.
\end{proof}

\subsection{Proof of Theorem~\ref{thm:grid_pooled}}
\begin{proof}
Fix a realized path and write \(u=r_t^{\mathcal G}\). By \eqref{eq:grid_cdf_error}, every \(g\in\mathcal G_M\) satisfies
\[
|\widehat F_{\mathcal B_t}(g)-\overline F_t(g)|\le u,\qquad
|\widehat F_{\mathcal B_t}(g^-)-\overline F_t(g^-)|\le u.
\]
The weighted block identity in the proof of Theorem~\ref{thm:drift} also gives \(u\le\varepsilon_t\).

Let \(\widetilde q_{\mathrm{lo},t}\) and \(\widetilde q_{\mathrm{hi},t}\) be the outward grid endpoints. If the lower endpoint is the sentinel zero, its strict lower-tail probability is zero. Otherwise it is a grid point no larger than \(q_{\mathrm{lo},t}\), and the empirical-inverse definition gives
\[
\widehat F_{\mathcal B_t}(\widetilde q_{\mathrm{lo},t}^-)
\le\alpha_{\mathrm{lo}}.
\]
Consequently,
\[
F_t(\widetilde q_{\mathrm{lo},t}^-)
\le
\alpha_{\mathrm{lo}}+u+D_{t,\mathcal G}.
\]
If the upper endpoint is \(+\infty\), its upper-tail probability is zero. Otherwise it is a grid point no smaller than \(q_{\mathrm{hi},t}\), so
\[
\widehat F_{\mathcal B_t}(\widetilde q_{\mathrm{hi},t})
\ge1-\alpha_{\mathrm{hi}},
\]
and hence
\[
1-F_t(\widetilde q_{\mathrm{hi},t})
\le
\alpha_{\mathrm{hi}}+u+D_{t,\mathcal G}.
\]
Adding the disjoint strict tails gives \(\alpha+2u+2D_{t,\mathcal G}\). Anchoring further enlarges the interval, and capping a probability by one proves \eqref{eq:grid_pooled_control}.
\end{proof}

\subsection{Proof of Corollary~\ref{cor:window_tradeoff}}
\begin{proof}
With equal block sizes, \(n_t=mK\) and the block weights equal \(1/K\). Therefore
\[
D_{t,\mathcal G}
\le D_t
\le
\frac1K\sum_{r=1}^{K}\|F_{t-r}-F_t\|_\infty
\le
\frac{\rho}{K}\sum_{r=1}^{K}r
=
\frac{\rho(K+1)}{2}.
\]
Substitution into Theorem~\ref{thm:grid_pooled} proves \eqref{eq:window_tradeoff}. Minimizing its excess term over the available candidate windows gives \eqref{eq:oracle_window}; no monotonicity of \(r_{t,K}^{\mathcal G}\) is used.
\end{proof}

\subsection{Proof of Proposition~\ref{prop:dynamic_tracking}}
\begin{proof}
Fix a path satisfying the stated conditions. By the definitions of the grid discrepancy and drift, every finite grid point has both strict and non-strict CDF error relative to the step mixture \(F_t\) at most
\[
\eta_t=r_t^{\mathcal G}+D_t.
\]
Fix a target \(p\), write \(q_t^\star=q_{t,p}^\star\), and choose \(\eta_t/\underline f<u<r_q\), which is possible by Assumption~\ref{assump:quantile_identifiability}. Set \(z_-=q_t^\star-u\) and \(z_+=q_t^\star+u\). The local CDF condition gives \(F_t(z_-)\le p-\underline f u\) and \(F_t(z_+)\ge p+\underline f u\). At the bracketing grid points, monotonicity and the grid-error bound imply \(\widehat F_{\mathcal B_t}(g^-(z_-))<p\) and \(\widehat F_{\mathcal B_t}(g^+(z_+))>p\). Hence the empirical quantile and either outward-rounded endpoint lie between these grid points. Their distance from \(q_t^\star\) is at most \(u+h_{\mathcal G}\). Letting \(u\downarrow\eta_t/\underline f\) yields
\begin{equation}
|\widetilde q_{t,p}-q_{t,p}^\star|
\le
h_{\mathcal G}+\frac{r_t^{\mathcal G}+D_t}{\underline f}.
\label{eq:proof_pointwise_quantile}
\end{equation}

For a window of the preceding \(K\) equal-weight CDFs,
\[
D_t
\le
\frac1K\sum_{r=1}^{K}\|F_{t-r}-F_t\|_\infty
\le
\frac1K\sum_{r=1}^{K}
\sum_{s=t-r+1}^{t}v_s.
\]
After summing over \(t\), each \(v_s\) receives total weight at most \(K^{-1}\sum_{r=1}^{K}r=(K+1)/2\). Thus
\[
\sum_{t=K+1}^{T}D_t
\le\frac{K+1}{2}V_T.
\]
Summing \eqref{eq:proof_pointwise_quantile} and upper-bounding each realized grid discrepancy by \(r_{T,K}^{\mathcal G}\) proves \eqref{eq:dynamic_quantile_tracking}.

The pinball risk is minimized at \(q_{t,p}^\star\), and its subgradients in \(q\) have absolute value at most one. Its excess is therefore at most \(|\widetilde q_{t,p}-q_{t,p}^\star|\), which proves the regret statement. Anchoring maps \(u\mapsto\min\{u,1\}\) and \(u\mapsto\max\{u,1\}\) are one-Lipschitz. Finally, apply the exact hinge representation in \eqref{eq:exact_objective_distortion} at each collected update and average with the fixed weights \(\nu_{t,k}\). This gives
\[
|\mathfrak D_t(a,b)-\mathfrak D_t(a',b')|
\le
A_{\max}(|a-a'|+|b-b'|).
\]
Applying the endpoint bound to both tails and summing proves \eqref{eq:dynamic_distortion_tracking}.
\end{proof}

\subsection{Proof of Proposition~\ref{prop:path_deviation}}
\begin{proof}
Let \(d_k=\|\vtheta_k^{\mathrm C}-\vtheta_k^{\mathrm G}\|\). Adding and subtracting \(g_{\mathrm G}(\vtheta_k^{\mathrm C})\), then using Lipschitzness and the assumed clipping bias, gives
\[
d_{k+1}
\le
(1+\gamma_kL_g)d_k+\gamma_kB_k.
\]
Since \(d_0=0\), iterating this recursion yields \eqref{eq:path_deviation}. The stated choice of \(B_k\) follows directly from Proposition~\ref{prop:gradient_fidelity} when its assumptions and the required pathwise tail bound hold at iterate \(k\).
\end{proof}

\newpage
\section{Supplementary Results}
\subsection{Detailed Evaluation Results}
\label{app:detailed_evaluation_results}



\begin{figure*}[htbp]
\centering
\includegraphics[width=0.6\textwidth]{figure/alfworld_legend.pdf}
\\
\subfloat[Pick.]{\includegraphics[width=0.3\textwidth]{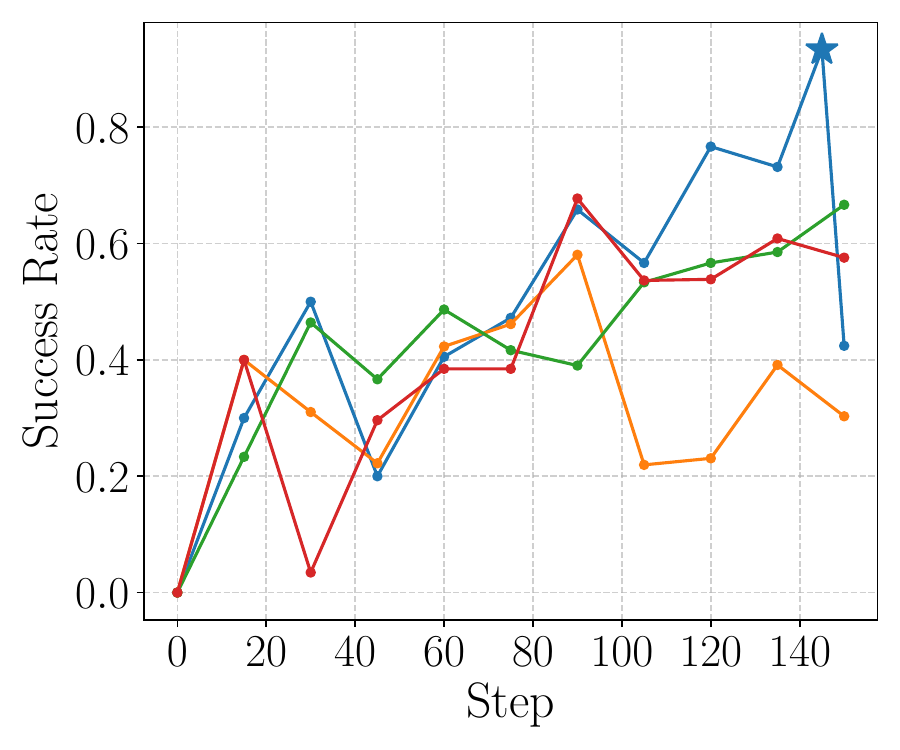}}
\subfloat[Look.]{\includegraphics[width=0.3\textwidth]{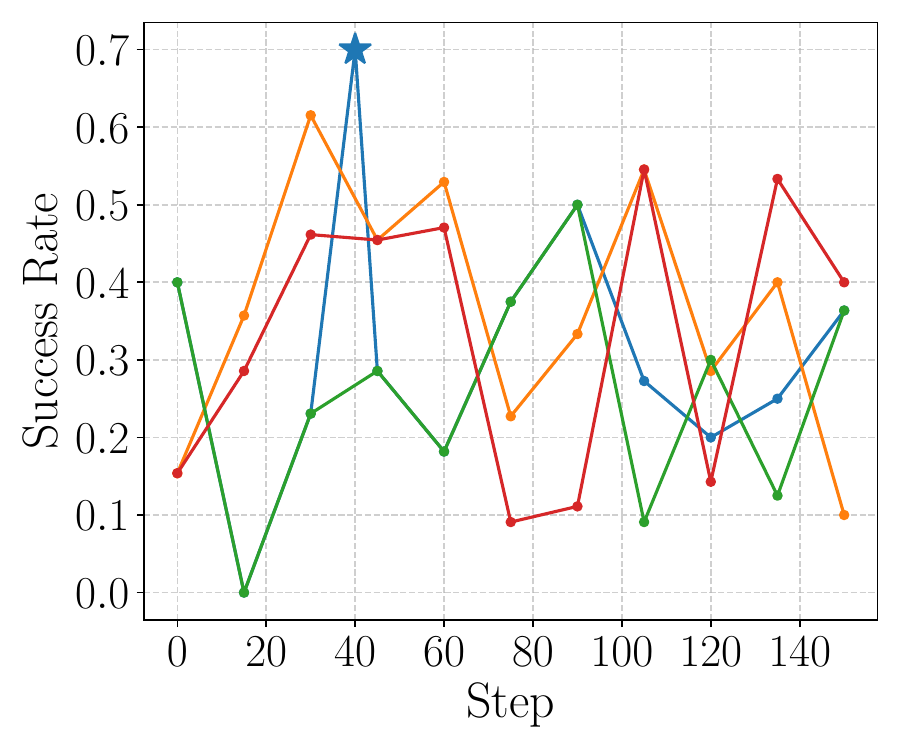}}
\subfloat[Clean.]{\includegraphics[width=0.3\textwidth]{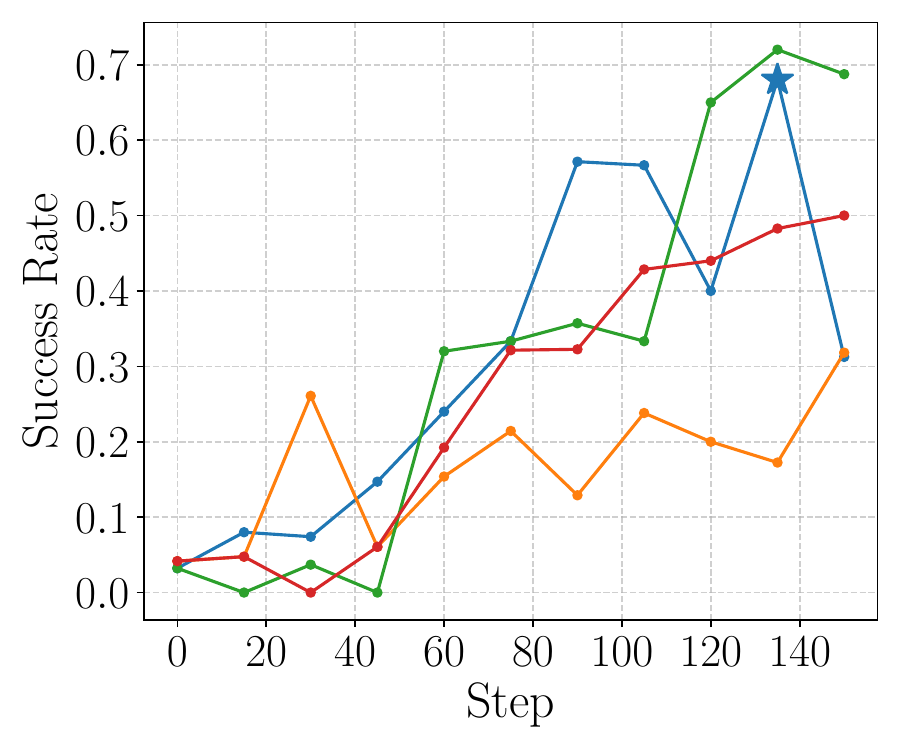}}
\\
\subfloat[Heat.]{\includegraphics[width=0.3\textwidth]{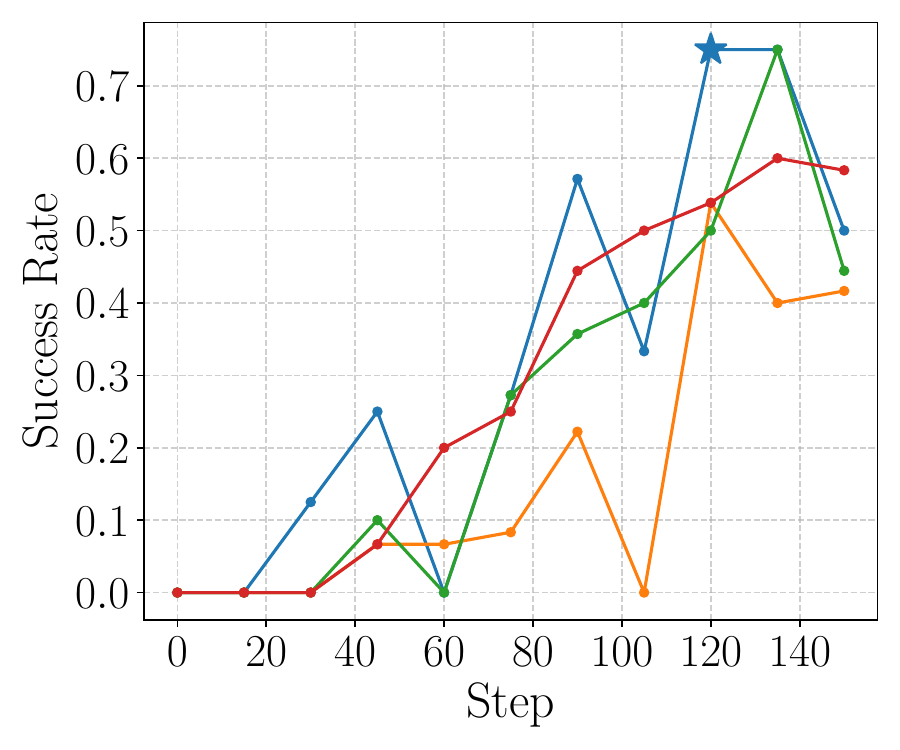}}
\subfloat[Cool.]{\includegraphics[width=0.3\textwidth]{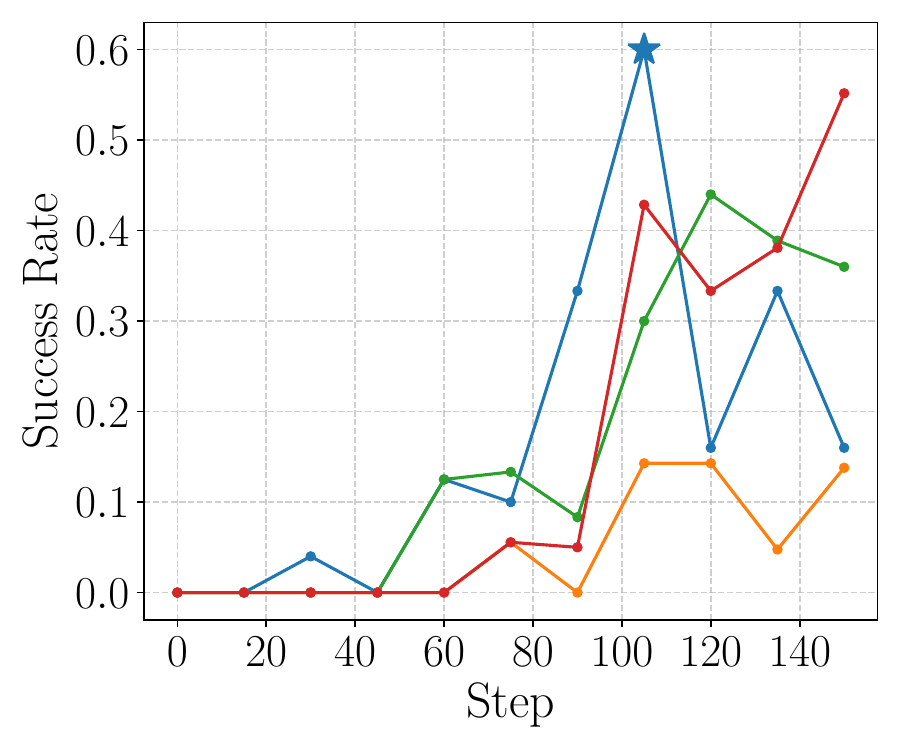}}
\subfloat[Pick2.]{\includegraphics[width=0.3\textwidth]{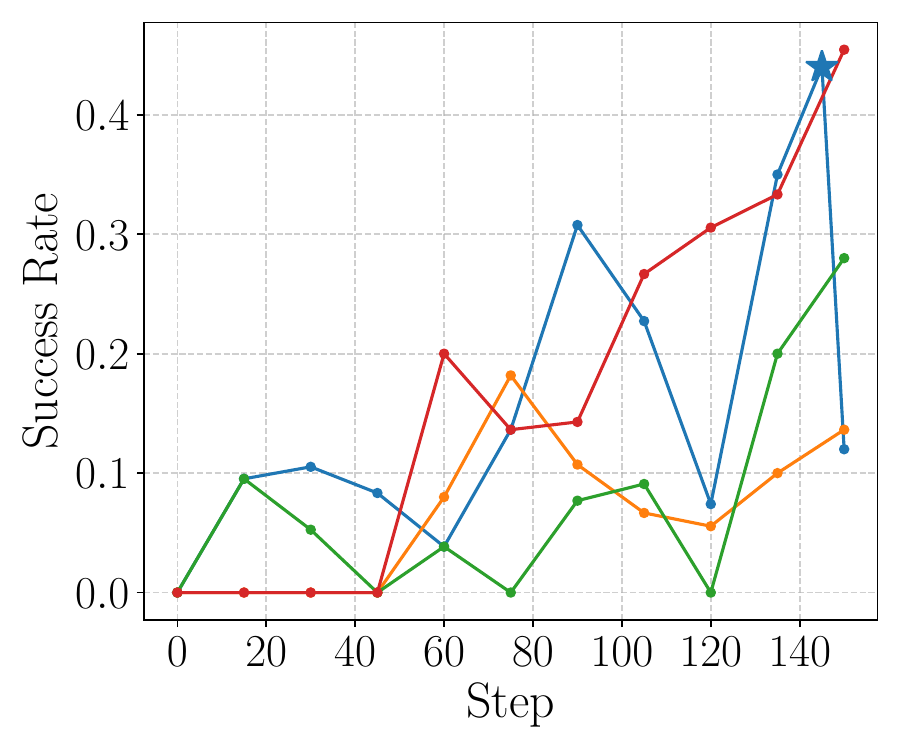}}
\caption{Validation results on ALFWorld.}
\label{fig:alfworld_validation_success}
\end{figure*}

\begin{figure*}[htbp]
\centering
\subfloat[AIME24.]{\includegraphics[width=0.3\textwidth]{figure/aime2024_pass_at_1.pdf}}
\subfloat[AIME25.]{\includegraphics[width=0.3\textwidth]{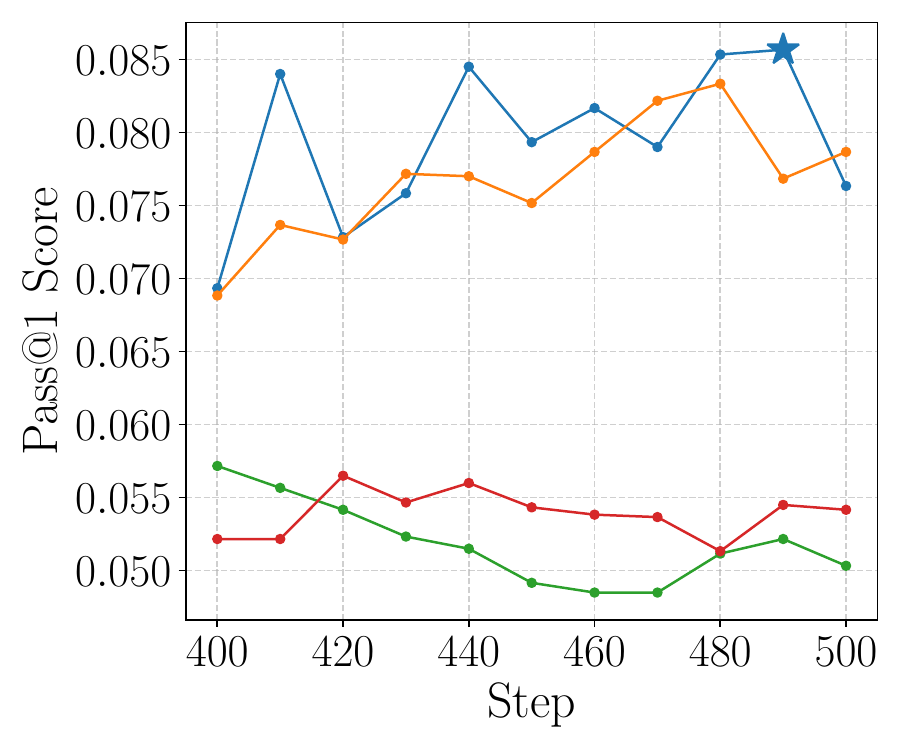}}
\subfloat[AMC.]{\includegraphics[width=0.3\textwidth]{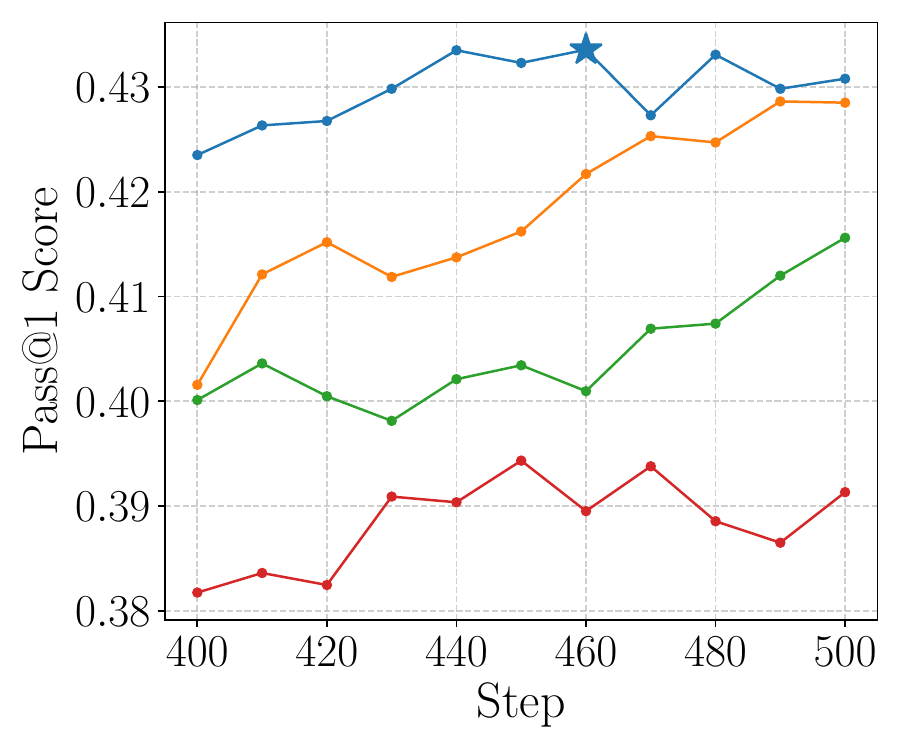}}
\\
\subfloat[MATH500.]{\includegraphics[width=0.3\textwidth]{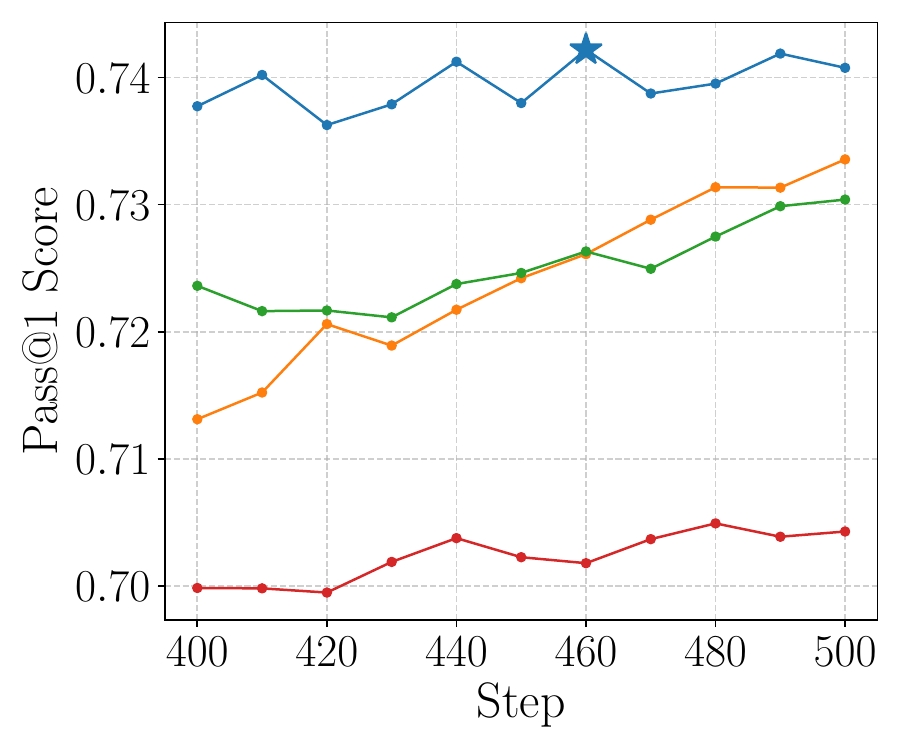}}
\subfloat[Minerva.]{\includegraphics[width=0.3\textwidth]{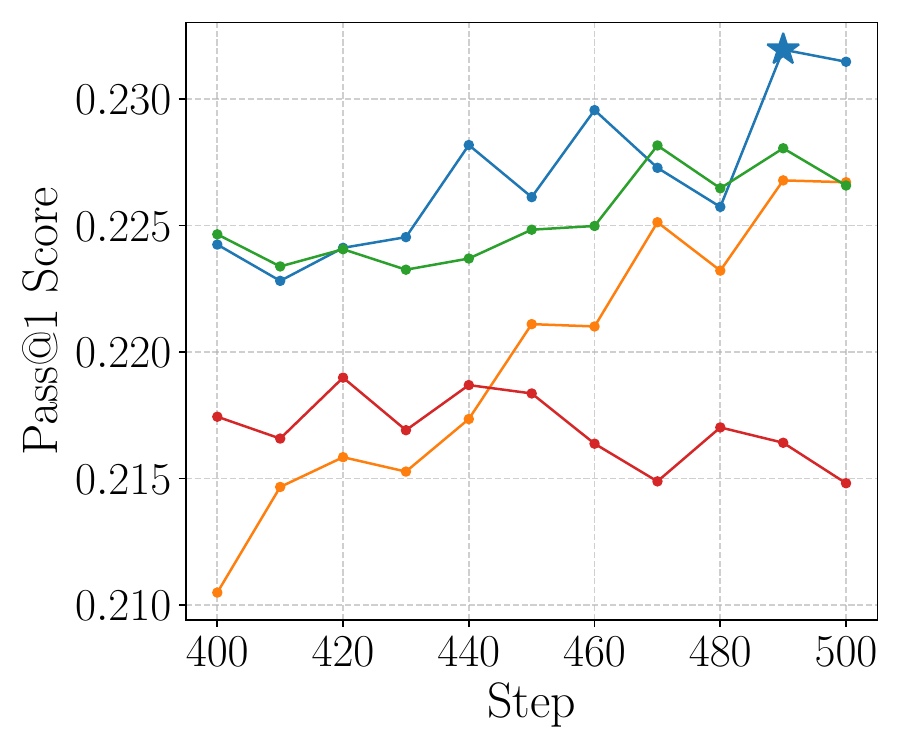}}
\subfloat[OlympiadBench.]{\includegraphics[width=0.3\textwidth]{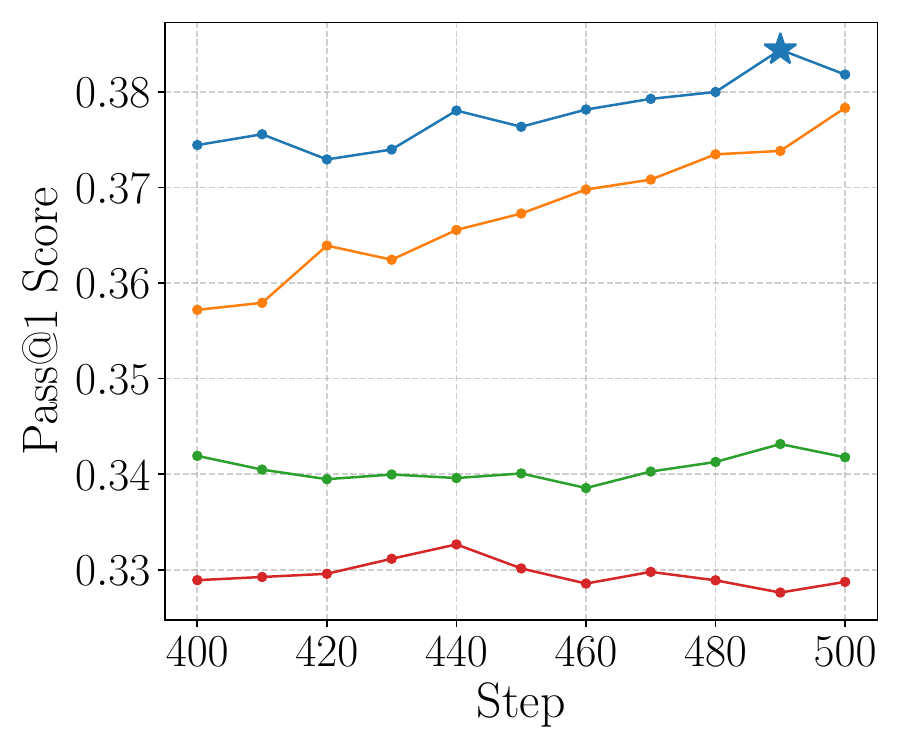}}
\caption{Pass@1 on various benchmarks.}
\label{fig:pass_at_1}
\end{figure*}

\subsection{Experimental setup}
\label{app:experimental_setup}

For ALFWorld, we use Qwen3-1.7B and roll out \(G=8\) trajectories per task instance, with a maximum of \(50\) interaction steps per episode. The maximum prompt and response lengths are \(2048\) and \(512\) tokens, respectively. We use a training batch size of \(16\) task instances and a mini-batch size of \(128\). All methods are trained for \(150\) epochs with a learning rate of \(1\text{e-}6\) and a KL penalty coefficient of \(0.01\).

For DAPO-Math-17k, we use Qwen3-1.7B-Base and roll out \(G=8\) responses per prompt, with maximum prompt and response lengths of \(2048\) and \(8192\) tokens, respectively. We use a training batch size of \(512\) prompts and a mini-batch size of \(128\), and train all methods for \(500\) steps with a learning rate of \(1\text{e-}6\). The learning-rate warmup ratio is \(0.05\) and the weight decay is \(0.1\). 

\subsection{Ablation Studies}
\label{app:ablation_studies}

In the main experiments presented in Section~\ref{subsec:main_results}, CSPO uses \(\alpha=0.1\) and \(K=2\) as its key hyperparameters. These choices were informed by preliminary ablation studies on mathematical reasoning tasks, including MATH \citep{hendrycks2021math} and GSM8K \citep{cobbe2021training}, following their original train-test splits. As shown in Table~\ref{tab:ablation_k_alpha}, performance is relatively stable across a range of \(K\) and \(\alpha\), while \(K=2\) and \(\alpha=0.1\) provide a strong and balanced configuration across the two benchmarks. We therefore adopt this setting for the main experiments.

\begin{table}[htbp] \centering \caption{Ablation study on window $K$ and target tail budget $\alpha$. Each column is a $(K, \alpha)$ configuration.} \adjustbox{width=0.95\textwidth}{ \begin{tabular}{l|c|cccc|ccc} \toprule $(K, \alpha)$ & $(2, 0.1)$ & $(1,0.1)$ & $(3,0.1)$ & $(4,0.1)$ & $(5,0.1)$ & $(2,0.05)$ & $(2,0.15)$ & $(2,0.2)$ \\ \midrule MATH & 65.32 & 65.82 & 65.30 & 66.04 & 64.68 & 64.48 & 65.60 & 65.16 \\ GSM8K & 86.35 & 85.82 & 85.44 & 85.97 & 85.60 & 86.20 & 86.05 & 85.90 \\ Avg. & 75.84 & 75.82 & 75.37 & 76.01 & 75.14 & 75.34 & 75.83 & 75.53 \\ \bottomrule \end{tabular} } \label{tab:ablation_k_alpha} \end{table}

We further study the effect of the miscoverage level \(\alpha\) on DAPO-Math-17k. As shown in Table~\ref{tab:hard_alpha_ablation}, CSPO generally achieves better performance with a smaller \(\alpha\), with \(\alpha=0.1\) obtaining the best average score. A smaller \(\alpha\) corresponds to a wider conformal clipping interval and therefore allows more policy updates to remain unclipped. However, \(\alpha\) should not be chosen arbitrarily small: as \(\alpha\) approaches zero, the clipping interval becomes increasingly loose and CSPO gradually approaches an effectively unclipped policy update, which departs from the trust-region intuition underlying TRPO and PPO and may increase the risk of unstable or collapsed training. Figure~\ref{fig:ablations_dapo_math} shows the rewards, clip fractions, and endpoints for these choices of \(\alpha\). The clip fraction is an empirical training diagnostic; interpreting it as a population tail probability requires accounting for data reuse and distinguishing band exceedance from saturation, which also depends on the advantage sign.

\begin{table}[htbp]
\centering
\caption{Ablations on miscoverage level \(\alpha\).}
\label{tab:hard_alpha_ablation}
\begin{tabular}{lccccccc}
\toprule
Method & AIME24 & AIME25 & AMC & MATH500 & Minerva & Oly. & Average \\
\midrule
CSPO-\(\alpha\)-0.1 & {16.87} & {8.57} & {43.36} & {74.22} & {23.19} & {38.44} & {34.11} \\
CSPO-\(\alpha\)-0.2 & 14.85 & 8.58 & 42.86 & 74.49 & 23.02 & 37.23 & 33.50 \\
CSPO-\(\alpha\)-0.4 & 13.52 & 7.40 & 41.50 & 72.98 & 22.47 & 36.34 & 32.37 \\
\bottomrule
\end{tabular}
\end{table}

\begin{figure*}[htbp]
\centering
\subfloat[Reward curves.]{\centering\includegraphics[trim=0cm 0.5cm 0cm 0cm,width=0.245\textwidth]{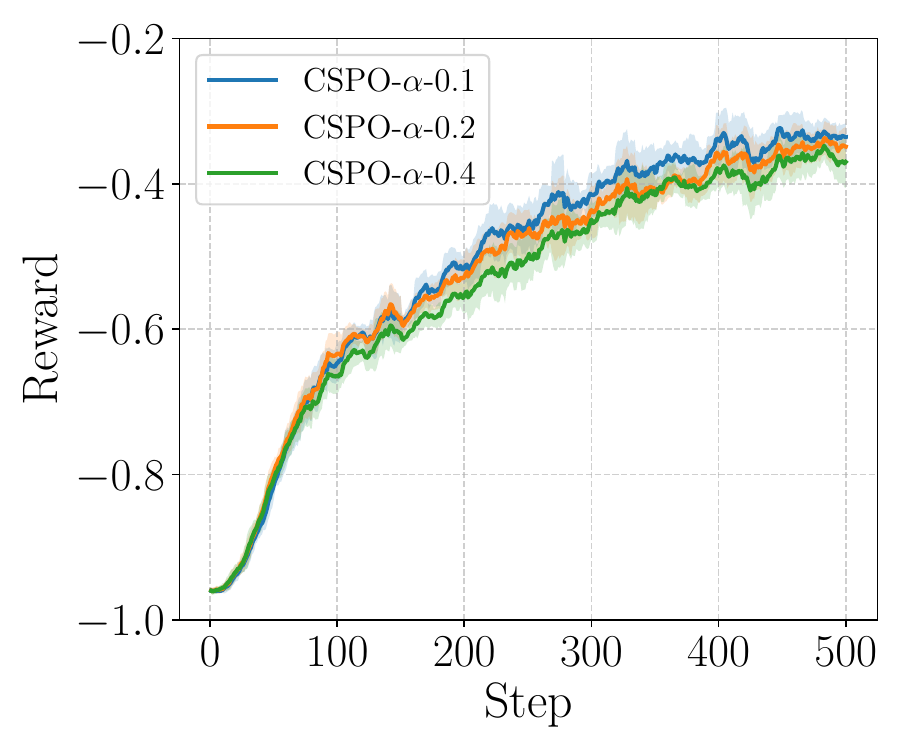}}
\subfloat[Clip fraction.]{\centering\includegraphics[trim=0cm 0.5cm 0cm 0cm,width=0.245\textwidth]{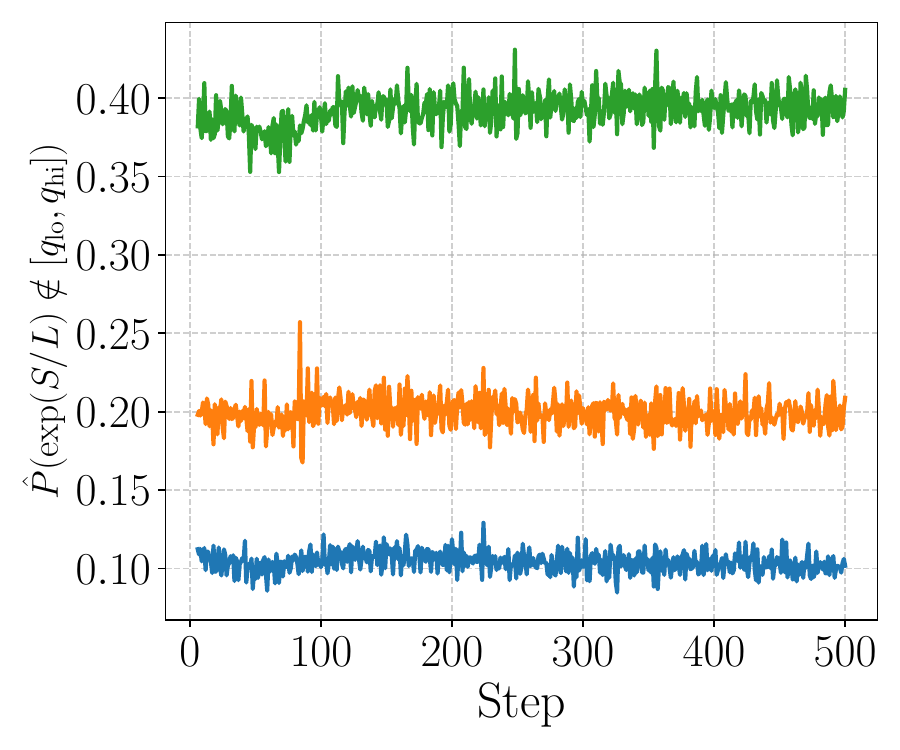}}
\subfloat[Dynamics of \(q_{\mathrm{hi}}\).]{\centering\includegraphics[trim=0cm 0.5cm 0cm 0cm,width=0.245\textwidth]{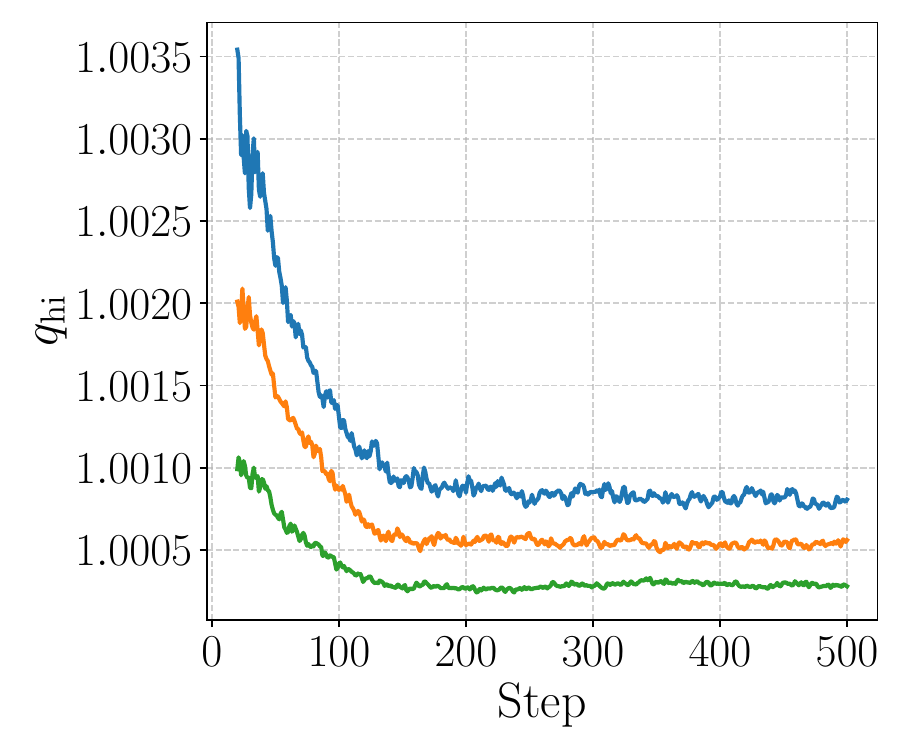}}
\subfloat[Dynamics of \(q_{\mathrm{lo}}\).]{\centering\includegraphics[trim=0cm 0.5cm 0cm 0cm,width=0.245\textwidth]{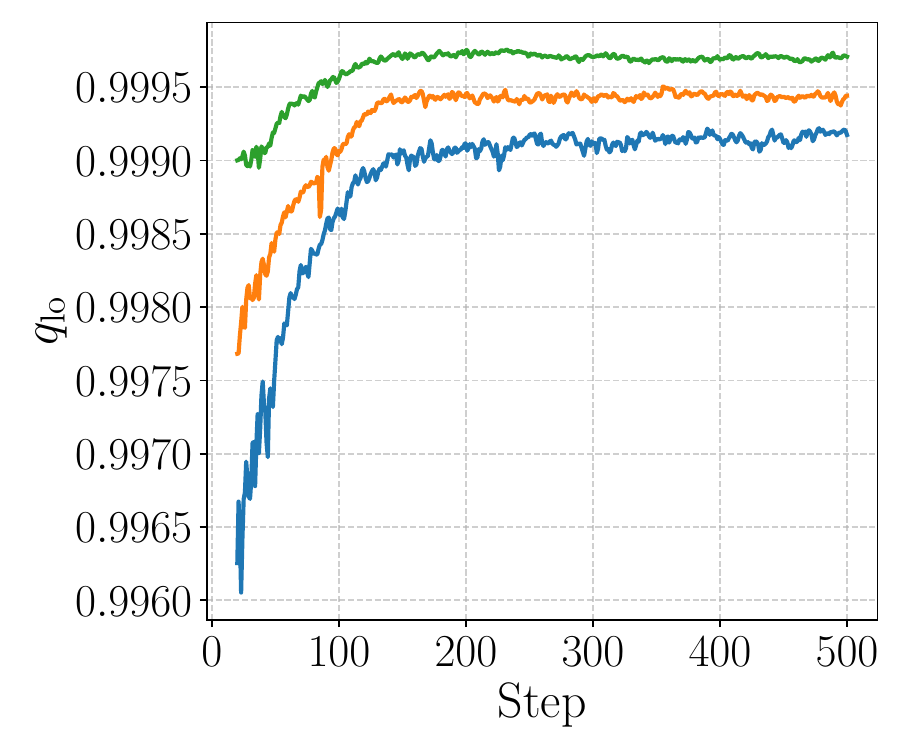}}
\caption{Ablation studies on DAPO-Math-17k.}
\label{fig:ablations_dapo_math}
\end{figure*}

\subsection{Additional Training Analysis}
\label{app:additional_training_analysis}
Figure~\ref{fig:dapo_math_efficiency} shows comparable training-step and actor-update times for CSPO and the baselines, indicating little additional overhead in this setting. Figure~\ref{fig:dapo_math_ppo_kl} reports the PPO KL diagnostic across methods and choices of \(\alpha\). These measurements describe policy changes during training, but do not directly verify the CDF drift assumption in \eqref{eq:window_drift}.

\begin{figure*}[htbp]
\centering
\subfloat[Time per training step.]{\includegraphics[width=0.45\textwidth]{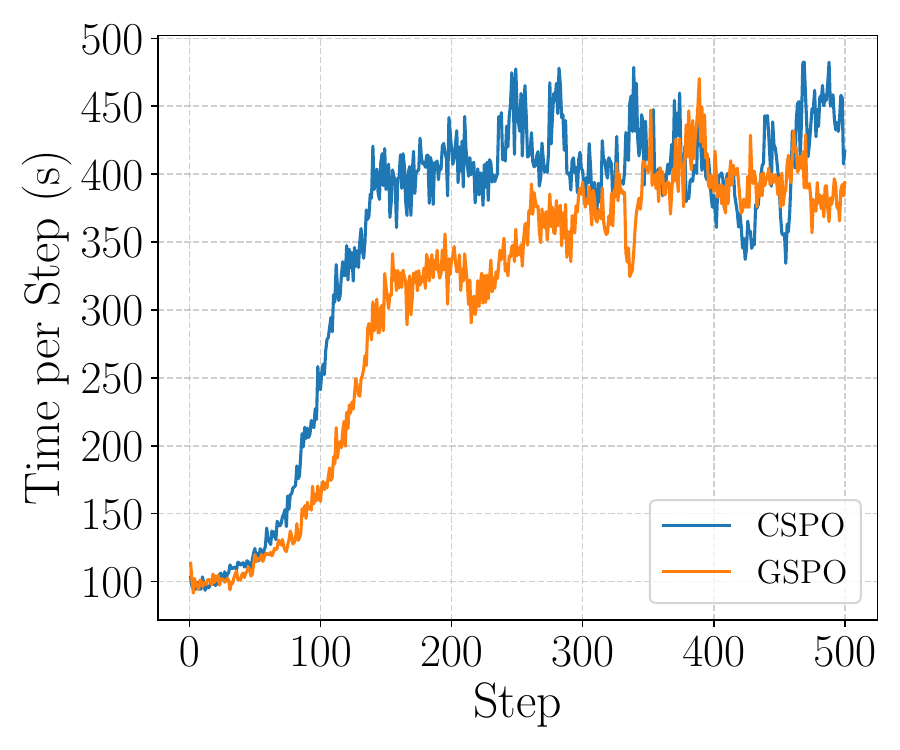}}
\subfloat[Actor update time.]{\includegraphics[width=0.45\textwidth]{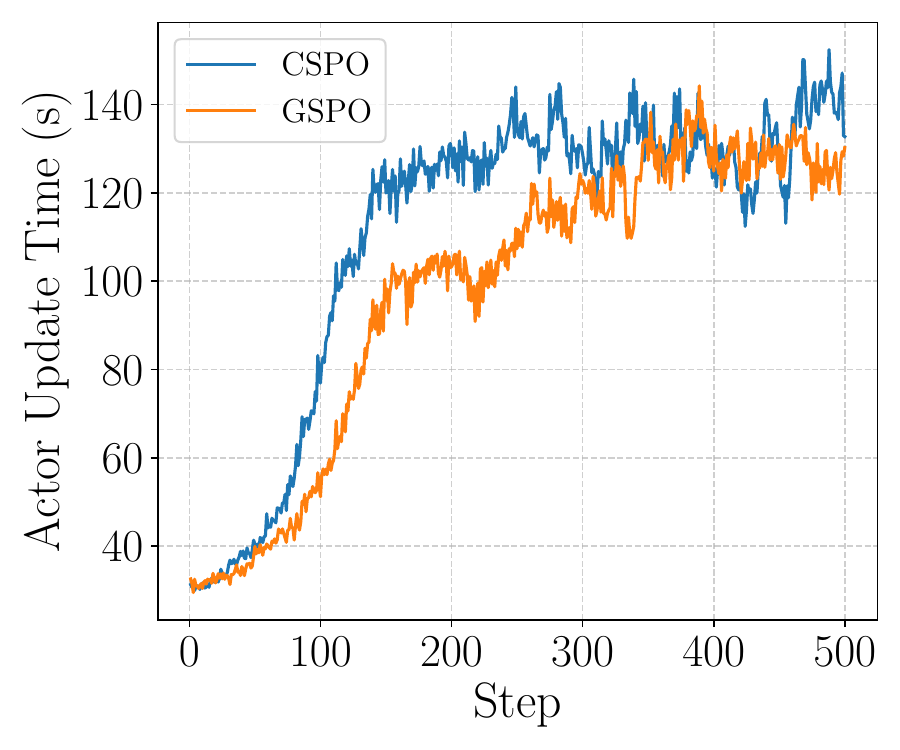}}
\caption{Computational efficiency on DAPO-Math-17k.}
\label{fig:dapo_math_efficiency}
\end{figure*}

\begin{figure*}[htbp]
\centering
\subfloat[PPO KL across methods.]{\includegraphics[width=0.45\textwidth]{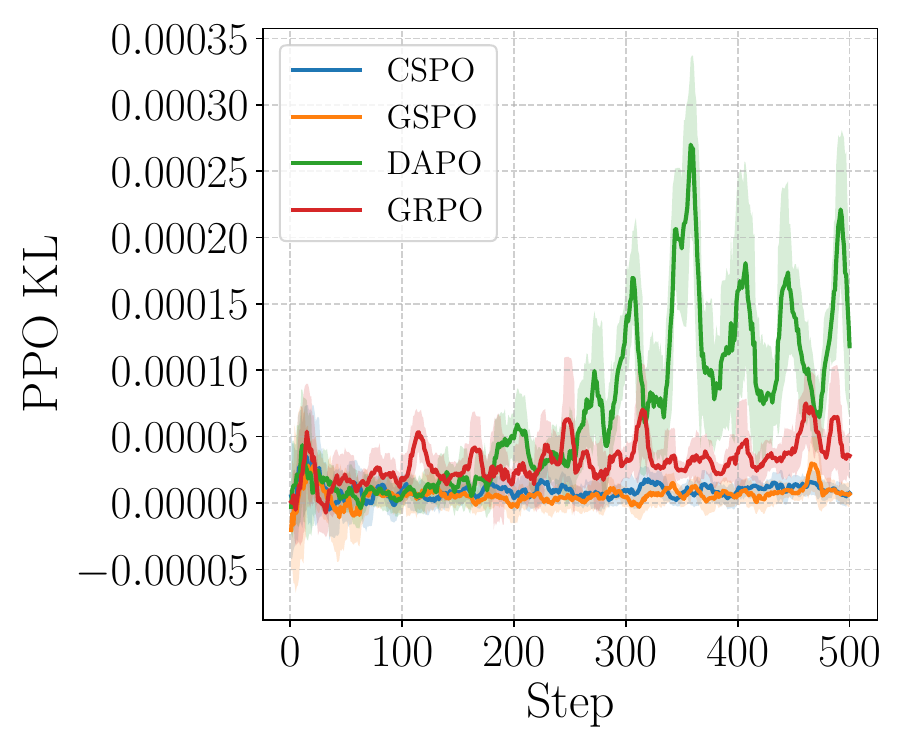}}
\subfloat[PPO KL across different \(\alpha\).]{\includegraphics[width=0.45\textwidth]{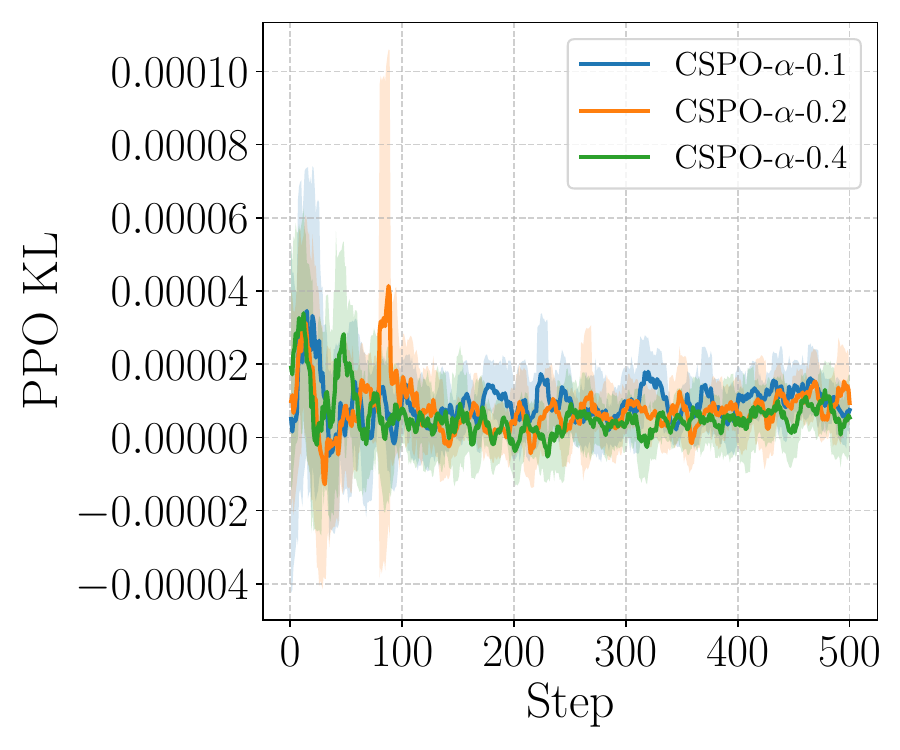}}
\caption{Policy update stability on DAPO-Math-17k.}
\label{fig:dapo_math_ppo_kl}
\end{figure*}

\end{document}